\documentclass{article}

\usepackage{fullpage}
\usepackage{euler}
\usepackage{charter}

\usepackage{amsmath}
\usepackage{amssymb}
\usepackage{amsthm}

\usepackage{graphicx}
\usepackage{balance}
\usepackage{subfigure}
\usepackage{xspace}
\usepackage[ruled,vlined]{algorithm2e}
\usepackage[usenames, dvipsnames]{color}
\usepackage{pdfpages}
\usepackage{xcolor}
\usepackage{listings}
\usepackage{url}
\usepackage{hyperref}
\usepackage[shortlabels]{enumitem}
\setitemize{noitemsep,topsep=0pt,parsep=0pt,partopsep=0pt}

\newcommand{\eml}{\texttt{ease.ml}\xspace}
\newcommand{\meter}{\texttt{ease.ml/meter}\xspace}

\usepackage{accents}
\newcommand{\ubar}[1]{\underaccent{\bar}{#1}}

\newcommand{\fh}[1]{\textcolor{red} {fh: #1}}

\newtheorem{theorem}{Theorem}

\newtheorem{definition}{Definition}
\newtheorem{example}{Example}
\newtheorem{corollary}{Corollary}

\DeclareMathOperator{\OVFT}{OVFT}

\usepackage{graphicx}
\usepackage{caption}

\begin{document}

\title{Quantitative Overfitting Management for Human-in-the-loop ML Application Development with ease.ml/meter\\
\large {\em Towards Data Management for Statistical Generalization}}

\author{
Frances Ann Hubis$^\dagger$~~~~~Wentao Wu$^\ddagger$~~~~~Ce Zhang$^\dagger$\\
$^\dagger$ETH Zurich~~~~~~~$^\ddagger$Microsoft Research, Redmond\\
\{hubisf, ce.zhang\}@inf.ethz.ch~~~~~wentao.wu@microsoft.com
}

\maketitle

\begin{abstract}
Simplifying machine learning (ML) application development, including distributed
computation, programming interface, resource management, model selection, etc,
has attracted intensive interests recently. These research efforts have
significantly improved the {\em efficiency} and the {\em degree of automation}
of developing ML models.

In this paper, we take a first step in an orthogonal direction towards
\emph{automated quality management} for human-in-the-loop ML application
development. We build \meter, a system that can automatically detect and measure
the degree of \emph{overfitting} during the whole lifecycle of ML application
development. \meter~returns overfitting signals with strong probabilistic
guarantees, based on which developers can take appropriate actions.
In particular, \meter~provides principled guidelines to simple yet nontrivial
questions regarding desired validation and test data sizes, which are among
commonest questions raised by developers.
The fact that ML application development is typically a \emph{continuous}
procedure further worsens the situation: The validation and test data sets can
lose their statistical power quickly due to \emph{multiple accesses}, especially
in the presence of \emph{adaptive analysis}. \meter~addresses these challenges
by leveraging a collection of novel techniques and optimizations, resulting in
practically tractable data sizes without compromising the probabilistic
guarantees. We present the design and implementation details of \meter, as well
as detailed theoretical analysis and empirical evaluation of its effectiveness.
\end{abstract}

\section{Introduction}

The fast advancement of machine learning technologies has triggered tremendous
interest in their adoption in a large range of applications, both in science
and business. Developing robust machine learning (ML) solutions to real-world,
mission-critical applications, however, is challenging. It is a
\emph{continuous} procedure that requires many iterations of training, tuning,
validating, and testing various machine learning models before a good one can
be found, which is tedious, time-consuming, and error-prone. This dilemma has
inspired lots of recent work towards simplifying ML application development,
covering aspects such as distributed computation~\cite{mulee-parameterSever,mllib}, resource management~\cite{LiZLWZ18,SparksTHFJK15}, AutoML~\cite{googleCloudML,Northstar}, etc.

During the last couple of years, we have been working together with many
developers, most of whom are not computer science experts, in building a range
of scientific and business applications~\cite{ackermann2018using,girardi2018patient,schawinski2017generative,schawinski2018exploring,stark2018psfgan,su2018generative}
and in the meanwhile, observe the challenges that they are facing. On the
positive side, recent research on efficient model training and AutoML definitely
improves their productivity significantly. However, as training machine learning
model becomes less of an issue, new challenges arise --- in the iterative
development process of an ML model, users are left with a powerful tool but not
enough principled guidelines regarding many design decisions that cannot be
automated by AutoML.

\vspace{0.5em}
\noindent
{\bf (Motivating Example)} In this paper, we focus on two of the most commonly
asked questions from our users\footnote{As an anecdotal note, THE most popularly
asked question is actually ``{\em how large does my training set need to be?}''}
---

\begin{center}
\em (Q1) How large does my validation set need to be?
\end{center}

\begin{center}
\em
(Q2) How large does my test set need to be?
\end{center}

\noindent
Both questions essentially ask about the {\em generalization} property of these
two datasets --- they need to be large enough
such that they form a
{\em representative sample} of
the (unknown) underlying true
distribution. However,
giving meaningful answers to these two questions (rather than answers like ``as
large as possible'' or
``maybe a million'') is
not an easy task. The reasons
are three fold. First, the
answers depend on the {\em error tolerance}
of the ML application --- a
mission-critical application
definitely needs a larger
validation/test set.
Second,
the answers depend on {\em the history
of how these data were
used} --- when a
dataset (especially the validation set)
is used multiple times,
it loses its statistical power and
as a result its size relies on
the set of all historical operations ever conducted. Third, the answers need to be
practically feasible/affordable --- simply
applying standard concentration
inequalities can lead to answers
that require millions of (or more)
labels that may be intractable.

\begin{figure}
\centering
\includegraphics[width=0.6\textwidth]{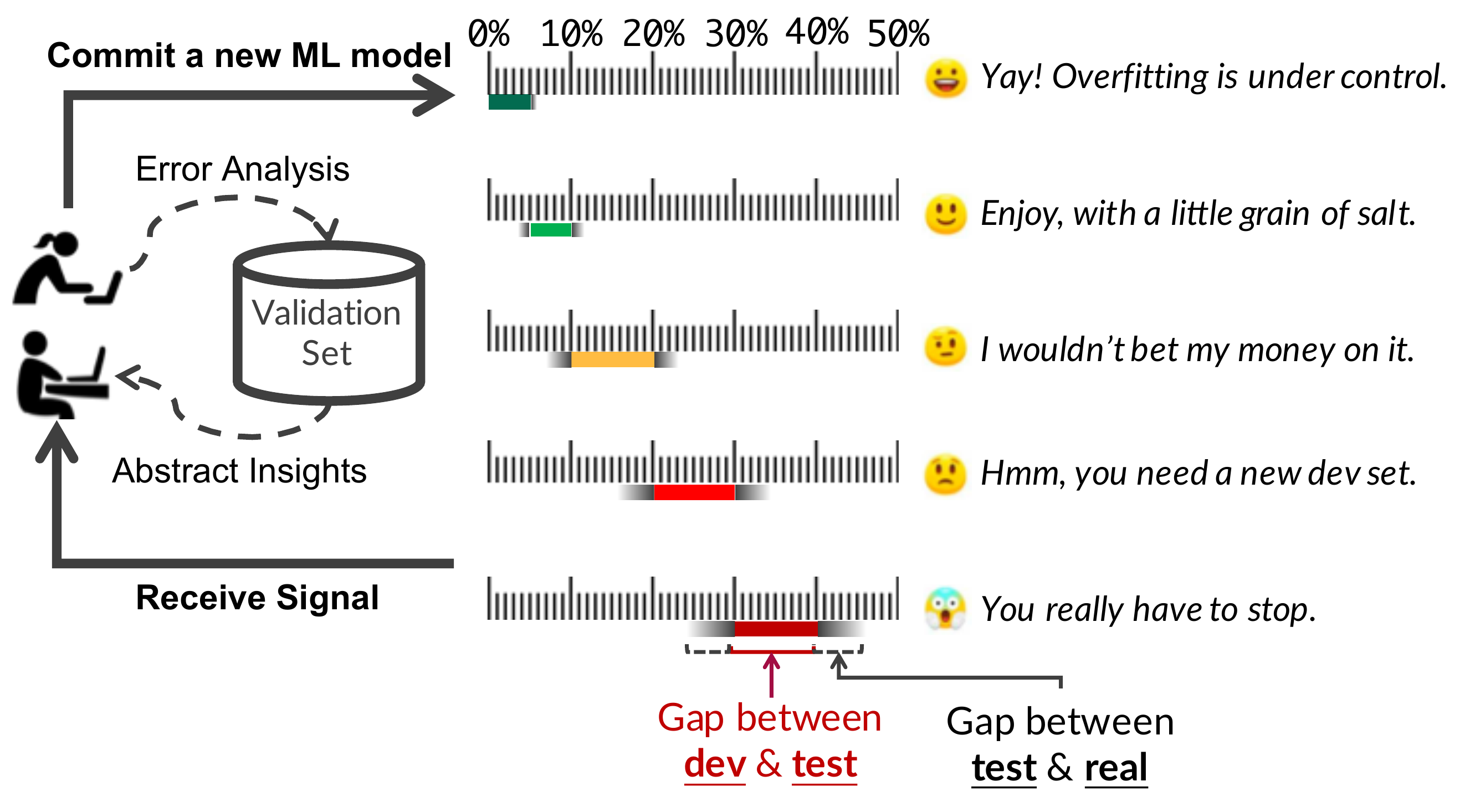}
\caption{Interaction with \meter.}
\label{fig:overview}
\end{figure}

\paragraph*{A Data Management System for Generalization}

In this paper, we present \meter,
a system that takes the first step in (not completely) tackling the above challenges.
Specifically, \meter~is a data management system designed to manage the statistical
power of the validation and test data sets.
Figure~\ref{fig:overview} illustrates
the functionality of \meter, which interacts
with its user in the following way:
\begin{enumerate}
\item The user inputs a set of
parameters specifying
the error tolerance of an ML application
(defined and scoped later).
\item The system returns $|D_{val}|$
and $|D_{test}|$, the required
sizes of the validation and test sets
that can satisfy the error tolerance.
\item The user provides a
validation set and a test set.
\item The user queries the validation
and test set. The system returns
the answer that satisfies the user-defined
error tolerance in Step 1. The system
monitors and constrains the usage of
the given validation set and test set.
\item When the system discovers that
the validation set or test set
loses its statistical power, it
asks the user for
another validation or test set
of size $|D_{val}|$
or $|D_{test}|$, respectively.
\end{enumerate}

\vspace{0.5em}
\noindent
{\bf (Scope)}
In this paper, we
focus on a very specific (yet typical)
special case of the above generic
interaction
framework --- the user
has \emph{full} access to the validation
set. She starts from the current
machine learning model, $H_i$,
conducts error analysis by looking
at the error that $H_i$ is
making on the validation set,
and proposes a new machine learning
model $H_{i+1}$ (e.g., by
adding new feature extractors,
trying a different ML model,
or adding more training examples).
However, as the validation
set is open to the user, after
many development iterations
the user's decision might
start to overfit to this specific
validation set.
Meanwhile, although the test set is hidden from the user,
the signals returned by the system inevitably carries \emph{some} information about the test set as well.
As a result, the user's decision might also overfit to this particular test set, which undermines its plausibility as a delegate of the underlying true data distribution.
{\em The goal of
our system is to (1) inform the
user whenever she needs to
use a new validation set
for error analysis, and (2) inform
the user whenever she needs to
use a new test set to measure
validation set's overfitting behavior.}

\vspace{0.5em}
\noindent
{\bf (System Overview)} In the
above workload, the
``error tolerance'' $\epsilon_{tot}$ is the generialization
power of the current validation set.
Specifically, let
$D_{test} \sim \mathcal{D}_{test}$
be the test set and $D_{val}$
be the validation set.
Let $H$ be
a machine learning model and
$l(H, -)$ returns the loss of $H$ on $D_{val}$, $D_{test}$,
and $\mathcal{D}_{test}$, respectively.\footnote{$l(H, \mathcal{D}_{test})$ represents the \emph{expected} loss over the true distribution, which is unknown.}
We assume that the user hopes to be
alerted whenever
\[
| l(H, D_{val}) - l(H, \mathcal{D}_{test}) | > \epsilon_{tot}.
\]
One critical design decision
in \meter is to decompose the LHS of the above inequality into two terms:
\begin{enumerate}
\item {\bf Empirical Overfitting:}
$|l(H, D_{val}) - l(H, D_{test})|$;
\item {\bf Distributional Overfitting:}
$|l(H, D_{test}) - l(H, \mathcal{D}_{test}))|$.
\end{enumerate}
The rationale behind this decomposition
is that it separates the roles
of the validation set and the
test set when overfitting occurs --- when the {\em empirical
overfitting} term is large,
the validation set ``diverges from''
the test set, and therefore,
the system should ask for a new validation set; when the {\em distributional overfitting} term is large,
the test set ``diverges from''
the true distribution, and therefore,
the system should ask for a new
test set.
Moreover, empirical overfitting can be directly computed, whereas distributional overfitting has to be estimated using nontrivial techniques, as we will see.

\meter~provides a ``\textbf{\emph{meter}}'' to
communicate the current level
of empirical
overfitting and distributional
overfitting to the user.
For each model $H_i$ the
user developed, the system
returns one out of five
possible signals as illustrated
in Figure~\ref{fig:overview} ---
the solid color bar represents
the range of the empirical
overfitting, and the gray color
bar represents the upper bound of
distributional overfitting.
{\em The user decides whether
to replace a validation set
according to the signals returned by the system, and
the system asks for a new
test set whenever it cannot guarantee
that the distributional overfitting
is smaller than a user-defined
upper bound, $\epsilon$.}

\paragraph*{Technical Challenges and Contributions}
The key technical challenge in
building \meter is to estimate distributional overfitting, which is nontrivial due to the fact
that subsequent versions of the ML application
(i.e., $H_i$)
can be \emph{dependent} on prior versions
(i.e., $H_1,...,H_{i-1}$).
As was demonstrated by recent work~\cite{DworkFHPRR15}, such kind of \emph{adaptivity} can accelerate the degradation of test set, fading its statistical power quickly.
To accommodate dependency between successive versions of the ML application, one may have to use a larger test set (compared with the case where all versions are \emph{independent} of each other) to provide the
same guarantee on distributional overfitting.

\vspace{0.5em}
\noindent
{\bf C1.} The first technical contribution of this work is to adapt techniques developed
recently by the theory and
ML community on
adaptive statistical querying~\cite{blum2015ladder, DworkFHPRR15}
to our scenario. Although
the underlying method (i.e., bounding
the description length) is the same,
it is the first time that such techniques
are applied to enabling a scenario that
is similar to what \meter tries to support.
Compared with the naive approach that
draws a new test set for each new version
of the application, the amount of test samples required by \meter~can be \emph{an order
of magnitude smaller}, which makes it more practical.

\vspace{0.5em}
\noindent
{\bf C2.}
The second technical contribution of this work is a set of simple, but novel optimizations that further reduce the amount of test samples required.
These optimizations go beyond traditional
adaptive analytics techniques~\cite{blum2015ladder, DworkFHPRR15} by taking into consideration
different specific operation modes that
\meter provides to its user, namely (1)
non-uniform error tolerance, (2) multi-tenant
isolation, (3) non-adversarial developers,
and (4) ``time travel.'' Each
of these techniques focuses on one specific
application scenario of \meter, and can
further reduce the
(expected) size of the
test set significantly.

\paragraph*{Relationship with Previous Work}
The most similar work to \meter is
our recent paper \texttt{ease.ml/ci}~\cite{CI}.
\texttt{ci} is a ``continuous integration''
system designed for the ML development process ---
given a new model provided by the user,
\texttt{ci} checks whether certain
statistical property holds (e.g.,
the new model is better than the old model,
tested with $(\epsilon, \delta)$). At a
very high level, \texttt{meter}
shares a similar goal as \texttt{ci}, however,
from the technical perspective, is
significantly more difficult to build,
for two reasons. First, \meter cannot
use the properties of the continuous integration
process (e.g., the new model will not change
too much) to optimize for the sample
complexity. As we will see, to achieve
practical sample
complexity, \texttt{meter} relies on a completely
different set of optimizations. Second,
\meter needs to support multiple signals,
instead of a binary pass/fail
signal as in \texttt{ci}.
As a result, we see \texttt{ci}
and \texttt{meter} fall into the
same ``conceptual umbrella'' of
{\em data management for
statistical generalization}, but focus
on different scenarios and thus
require different technical optimizations
and system designs.

\paragraph*{Limitations} We believe that
\meter~is an ``innovative system''
that provides functionalities that we
have not seen in current ML ecosystems.
Although \meter works reasonably well for
our target workload in this paper, it
has several limitations/assumptions
that require
future investment.
First, \meter~requires new test data points as well as their labels from developers.
Although recent work on automated labeling (e.g., Snorkel~\cite{RatnerBEFWR17}) alleviates the dependency on human labor for generating labeled training data points, it does not address our problem as we require labels for test data rather than training data.
Second, the question of what actions should be taken upon receiving overfitting signals is left to developers.
A specific reaction strategy can lead to a more specific type of adaptive analysis that may have further impact on reducing the size of the test set.
Moreover, while this work focuses on monitoring overfitting, there are other aspects regarding quality control in ML application development.
For instance, one may wish to ensure that there is no performance regression, i.e., the next version of the ML application must improve over the current version~\cite{CI}.
We believe that quality control and lifecycle management in ML application development is a promising area that deserves further research.


\section{Preliminaries} \label{sec:preliminaries}

The core of \meter is based on the theory of answering
adaptive statistical queries~\cite{DworkFHPRR15}. In this section, we start
by introducing the traditional data model in {\em supervised}
machine learning, and the existing theory of supporting
adaptive statistical queries which will serve as the baseline
that we will compare with in \meter. As we will see in
Section~\ref{sec:optimization}, with a collection of simple, but
novel optimization techniques, we are able to significantly bring
down the requirement of the number of human labels, sometimes by several orders of magnitude.

\subsection{Training, Validation, and Testing}

In the current version of \meter, we focus on the
supervised machine learning setting, which consists of three
data distribution: (1) the {\em training} distribution
$\mathcal{D}_{train}$, (2) the
{\em validation} distribution $\mathcal{D}_{val}$,
and (3) the {\em test} distribution
$\mathcal{D}_{test}$, each of which defines a probability distribution
$P_{\mathcal{D}}$ over $(x, y)$, where $x \in \mathbb{R}^d$
is the {\em feature vector} of dimension $d$
and $y \in \mathbb{R}$ is the {\em label}.

\vspace{0.5em}
\noindent
{\bf (Application Scenarios)}
In traditional supervised learning setting, one
often assumes that all three distributions are the same, i.e., $\mathcal{D}_{train} = \mathcal{D}_{val} =
\mathcal{D}_{test}$. In \meter, we intentionally distinguish
between these three distributions to incorporate {\em two} emerging
scenarios that we see from our users. First, in
{\em weakly supervised} learning paradigm such as
data programming~\cite{RatnerSWSR16} or distant
supervision~\cite{MintzBSJ09}, the training
distribution $\mathcal{D}_{train}$ is a {\em noisy} version
of the real distribution. As a result
$\mathcal{D}_{train} \ne \mathcal{D}_{test}$.
The second example is motivated by one
anomaly detection application we built together
with a telecommunication company, in which the
validation and training distributions are
{\em injected} with anomalies and the real distribution
is an empirical distribution composed of/from real anomalies collected over history.
As a result, $\mathcal{D}_{val} \ne \mathcal{D}_{test}$.
The functionality provided by
\meter does not rely on the assumption that these
three distributions are the same.

\vspace{0.5em}
\noindent
{\bf (Sampling from Distribution)}
When building ML applications, in many, if not all, cases
user does not have access to the data distributions. Instead, user only has access to a {\em finite set of samples}
from each distribution: $D_{train} = \{(x^{t}_i, y^{t}_i)\} \sim \mathcal{D}_{train}$,
$D_{val} = \{(x^{v}_i, y^{v}_i)\} \sim \mathcal{D}_{val}$,
and $D_{test} = \{(x^{r}_i, y^{r}_i)\} \sim \mathcal{D}_{test}$.
As noted in the introduction, one common question from
our users is: {\em How large does the
training/validation/test set need to be?} The goal of
\meter is to provide {one} way of deciding the test set size
$|D_{test}|$, as well as when to draw a new test set.

\subsection{Human-in-the-loop ML Development}

An ML application is a function
$H: \mathbb{R}^d \mapsto \mathbb{R}$ that maps a feature vector $x$ to its predicted
label $f(x)$. Coming up with this function is
not a one-shot process, as indicated in previous work~\cite{url2, url3, url1, url4, url5}.
Instead, it often involves
human developers who (1) start from
a baseline application $H_0$, (2) conduct
error analysis by looking at the prediction
of the current application $H_t$ and summarize a
taxonomy of errors the application is making,
and (3) try out a {\em ``fix''} to produce a
new application $H_{t+1}$. A potential
``fix'' could be (1) adding a new feature,
(2) using a different noise model of data,
and (3) using a different model family,
architecture, or hyperparamter.

There are different frameworks of modeling human behavior.
In this paper, we adopt one that is commonly used by previous work on
answering adaptive statistical queries~\cite{blum2015ladder, DworkFHPRR15, HardtU14, SteinkeU15}.
Specifically, we assume that the user,
at step $t$, is a \emph{deterministic} mapping $U$
that maps the current application $H_t$ into
\[
H_{t+1} := U(H_t, D_{val}, D_{train}, g(H_t, D_{test}), \xi_t)
\]
We explain the parameterization of $H_t$ in the following:

\begin{enumerate}
\item The first three parameters $H_t$, $D_{val}$, and $D_{train}$ captures the scenario that the human developer has full access to
the training and validation sets, as well as the current version $H_t$ of the
ML application, and can use
them to develop the next version $H_{t+1}$ of the application.
\item The fourth parameter $g(H_t, D_{test})$ captures the
scenario in which the human developer only
has limited access to the test set.
Here $g$ is a \emph{set function} mapping from $H_t$ and $D_{test}$ to a set of \emph{feedback} returned by \meter to the developer.
As a special case, if $g(\cdot,\cdot) = \emptyset$,
it models the scenario in which the developer does not
have access to the test set at all (i.e., does not have any feedback) during development.

\item The fifth parameter $\xi_t$ models the ``environment effect''
that is orthogonal to the developer.
$\xi_t$ is a variable that does not
rely on past decisions, and is only
a function of the step ID $t$.
\end{enumerate}

When it is clear from the context that
$D_{val}$, $D_{train}$,
and $\xi_t$ are given, we abbreviate the notation as
\[
H_{t+1} = U(H_t, g(H_t, D_{test})).
\]

\vspace{0.5em}
\noindent
{\bf (Limitations and Assumptions)}
There are multiple limitations
that are inherent to the above model of
human behaviours.
Some we believe are
fixable with techniques similar to what we
propose in this paper, while others are more
fundamental.
The most fundamental assumption is that {\em human decision
does not have impact on the environment}.
In other words, $H_{t+1}$ is only a function of $H_0$ and all past feedback signals
$g(H_t, D_{test})$. There are also
other potential extensions that one could
develop. For example, instead of treating human behavior as a deterministic function $U$, one can extend it to a class of deterministic functions following some (known or unknown) probabilistic distribution.

\subsection{Generalization and Test Set Size}

The goal of ML is to learn a model over
a finite sample that can be {\em generalized} to
the underlying distribution that the user does not
have access to. In this paper, we focus on
the following {\em loss function}
$l: \mathbb{R}^d \times \mathbb{R} \times \mathbb{R} \rightarrow \{0, 1\} , (x,y,f(x)) \mapsto l(x,y,H(x)).$
which maps each data point, along with its prediction,
to either 0 or 1. (We refer to this loss as ``0-1 loss.'') We also focus on the following
notion of ``generialization'' for a given model $H$ (\cite{UMLbook}):
\[
\Pr\left[ \left| \frac{1}{n} \sum_{D_{test}} l(x,y, H(x)) - \mathbb{E}_{\mathcal{D}_{test}} l(x, y, H(x)) \right| > \epsilon \right] < \delta,
\]
where $(x,y)\in D_{test}$ in the first $l(x,y,H(x))$ and $(x,y)\sim\mathcal{D}_{test}$ in the second $l(x,y,H(x))$.
Given an ML model $H$, {\em checking} whether $H$
generializes according to this definition
is simple if one only uses the test set $D_{test}$ \emph{once}. In this case,
one can simply apply Hoeffding's inequality (Appendix~\ref{appendix:theory:hoeffiding})
to obtain:
\[
|D_{test}| \ge \frac{\ln 2/\delta}{2\epsilon^2}.
\]
This provides a way of deciding the
required number of samples in the test set.
It becomes tricky, however, when the test set is
{\em used multiple times}, and the goal of
\meter is to automatically manage this scenario
and decrease the required size of $D_{test}$.

\subsection{Adaptive Analysis and Statistical Queries}\label{sec:preliminaries:adaptive}

In recent years, there is an emerging research field regarding the so-called {\em adaptive analysis} or {\em reusable holdout}~\cite{DworkFHPRR15} that focuses on ML scenarios
where the test set $D_{test}$ can be accessed multiple times.
In our setting, consider $T$ ML models $H_1$, ... $H_T$ where
\[
H_{t+1} = U(H_t, g(H_t, D_{test})).
\]
The goal is to make sure that
\[
\Pr\left[ \exists t, ~~ \left| \frac{1}{n} \sum_{D_{test}} l(x,y, f_t(x)) - \mathbb{E}_{\mathcal{D}} l(x, y, f_t(x)) \right| > \epsilon \right] < \delta.
\]
When $g(\cdot,\cdot)$ is non-trivial,
simply applying union bound and requiring a
test set of size (see Appendix~\ref{appendix:theory:sample-complexity-basics:independent} for details)
\begin{equation}\label{eq:independent}
|D_{test}| \ge \frac{\ln (2T/\delta)}{2\epsilon^2}
\end{equation}
does not provide the desired probabilistic guarantee because of the \emph{dependency} between $H_t$ and $H_{t+1}$.
We will discuss adaptive analysis in detail when we discuss the overfitting meter in Section~\ref{sec:meter}.

\vspace{0.5em}
\noindent
{\bf (Baseline: Resampling)}
If a new test set $D_{test}^{(t)}$ is sampled from the distribution
$\mathcal{D}_{test}$ in each step $t$, to make sure that
with probability $1-\delta$
{\em all} $T$ models return a generalized
loss, one only needs to apply union bound
to make sure that each sampled test set
generalizes with probability $1 - \delta/T$.
Thus, to support $T$ adaptive steps,
one needs a test set of size (see Appendix~\ref{appendix:theory:sample-complexity-basics:re-sampling} for details)
\begin{equation}\label{eq:resampling}
|D_{test}| \ge \frac{T \ln (2T/\delta)}{2\epsilon^2}.
\end{equation}
This gives us a simple baseline approach that
provides the above generalization guarantee.
Unfortunately, it usually requires
a huge amount of samples, as there is essentially
no ``reuse'' of the test set.

\section{System Design}\label{sec:sys}

We describe in detail (1) the interaction model
between a user and \meter, (2) different system
components of \meter, and (3) the formalization
of the guarantee that \meter provides.
Last but not least, we provide a concrete
example illustrating how \meter would
operate using real ML development traces.

\subsection{User Interactions} \label{sec:sys:syntax}

\meter assumes that there are two types of users ---
(1) developers, who develop ML models, and (2)
labelers, who can provide labels for data points in the validation set or the test set. We
assume that the developers and labelers do not
have offline communication that \meter is not aware of
(e.g., labelers cannot send developers the
test set via email).

\paragraph*{Access Control} The separation
between developers and labelers is to allow \meter to
manage the access of data:

\begin{enumerate}

  \item Developers have full access to the
  validation set;

  \item Developers have \emph{no} access to the
  test set (\meter encrypts the test set and
  only labelers can decrypt it);

  \item Labelers have full access to the
  test set;

  \item Labelers have full access to the validation
  set.

\end{enumerate}

The rationale for the above protocol is that
\meter can measure the amount of information
that is ``leaked'' from the test set to the developers,
which, as we will see, is the key to bounding the degree of overfitting over the test set.

\paragraph*{Interaction with Developers}
A meter (as illustrated in Figure~\ref{fig:example_meter})
is specified by a set of triples
\[
\{(\ubar{r}_i,
\bar{r}_i, \epsilon_i)\}_{i=1}^{m}
\text{ s.t. }\cup_i [\ubar{r}_i, \bar{r}_i] = 1;\text{ }
\forall i \ne j, [\ubar{r}_i, \bar{r}_i] \cap [\ubar{r}_j, \bar{r}_j] = \emptyset.
\]
Each $(\ubar{r}_i,
\bar{r}_i, \epsilon_i)$ defines one of the $m$ possible \emph{overfitting signal}s (e.g., $m=5$ in Figure~\ref{fig:overview}) returned by \meter
to the developer: (1) Each $[\ubar{r}_i,
\bar{r}_i]$ specifies the range of
{\em empirical overfitting}
that this signal covers (i.e., the
solid color bar in Figure~\ref{fig:example_meter});
and (2) $\epsilon_i$ specifies the
upper bound of {\em distributional
overfitting} that this signal
guarantees (i.e., the gray color bar
in Figure~\ref{fig:example_meter}).

We assume in \meter that
\[
\ubar{r}_i < \bar{r}_i = \ubar{r}_{i+1} < \bar{r}_{i+1}, \text{ and } \epsilon_1 \le \epsilon_2 \le ... \le \epsilon_m.
\]
The rationale behind the non-decreasing
${\epsilon_i}$ is because of the intuition
that when the empirical overfitting
is already quite large, the developer
often cares less about small distributional overfitting.

	\begin{figure}
	\centering
	\includegraphics[width=0.4\textwidth]{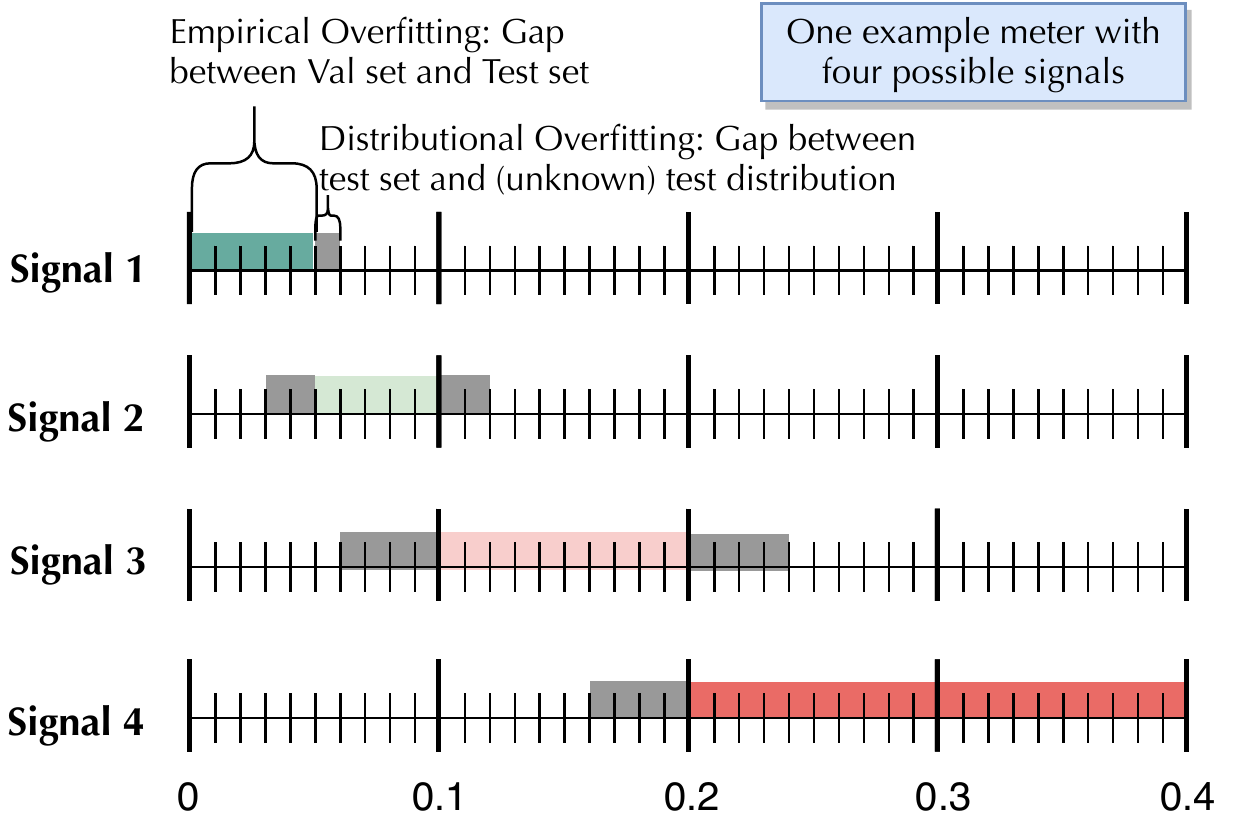}
	\caption{An example meter defined by four triples,
			i.e., four possible signals
			(see Section~\ref{sec:sys:syntax}):
			Signal $i$ is defined by
			$(\underline{r}_i, \bar{r}_i, \epsilon_i)$.
			If the system sends Signal 1 to the user, it means that the
			gap between the validation set accuracy and the test
			set accuracy falls into $[\underline{r}_i=0, \bar{r}_i=0.05]$,
			while the test set accuracy is at most $\epsilon_1 = 0.01$
			away from the (unknown) accuracy on the real data
			distribution.}
	\label{fig:example_meter}
	\end{figure}

\paragraph*{Interaction with Developers}
There are two phases when the developer interacts with \meter, {\em initialization}
and {\em model development}:

\begin{enumerate}

	\item {\em Initialization Phase (user initiates).} The developer
	initializes an \meter session by
	specifying the {\em length of development cycle}, $T$,
	the number of iterative development steps the developer hopes that this session can support;
	the developer also provides the {\em performance metric},
	$l$, a loss function (e.g., accuracy) whose output is bounded by $[0, 1]$.
	The developer further submits the current validation set
	to \meter.

	\item {\em Initialization Phase (system response).}
	Given $T$, the meter returns $|D_{test}|$, the
	number of examples required to support $T$
	development steps from the developer. The system
	will then request $|D_{test}|$ labels from the
	{\em labeler}.

\end{enumerate}

The developer starts development after the initialization phase:

\begin{enumerate}

	\item {\em Model Development Phase (developer initiates).} The developer
	submits the new ML model to \meter.

	\item {\em Model Development Phase (system response).} Given the
	model $H$, the system
	calculates the {\em empirical overfitting}, $|l(H, D_{val}) - l(H, D_{test})|$ (i.e., the
	gap between the losses over the validation and test sets) and finds the
	response $i$ s.t.
	\[
	\ubar{r}_i \leq |l(H, D_{val}) - l(H, D_{test})| \leq \bar{r}_i
	\]
	and returns the value $i$ to the developer.
	In the meantime, the system
	guarantees that the distributional
	overfitting, i.e.,
	\[
	|l(H, D_{test}) - l(H, \mathcal{D}_{test})|
	\]
	is smaller than $\epsilon_i$,
	with (high) probability $1 - \delta$.

	\item {\em Model Development Phase (developer)}. The developer
	receives the response $i$ and decodes it to $[\ubar{r}_i, \bar{r}_i]$.
	The developer then decides whether the empirical overfitting is
	too large. If so,
	she might choose to collect a new validation set.

\end{enumerate}

After $T$ development cycles (i.e., after the developer has checked in $T$ models)
the system terminates. The developer may then initiate a new \meter session.

\paragraph*{Interaction with Labelers}
Labelers are responsible for providing labeled test data.
Whenever the preset budget $T$ is used up, that is, the developer has
submitted $T$ versions of the ML application, \meter~issues a new request to the labeler to ask for a new, independent test set.
The old test set can be replaced by and released to the developer for development use.

\subsection{Overfitting Signals}\label{sec:sys:semantics}

In standard ML settings, overfitting is connected with model training -- not testing.
This is due to the (implicit) assumption that the test set will only be accessed
\emph{once} by the ML model.
In the context of continuous ML application development, this assumption is no
longer valid and the test set is subject to overfitting as well.
The presence of adaptive analysis further accelerates the process towards overfitting.
We now formally define the semantics
of the overfitting signals returned by \meter. Without loss of generality, we
assume that $\epsilon_1 = ... = \epsilon_m = \epsilon$
and only discuss the case of them being
different in Section~\ref{sec:optimization:non-uniform}.

\subsubsection{Formalization of User Behavior}\label{sec:sys:semantics:behavior}

To formalize this notion of overfitting, we need precise characterization of the
behavior of the developer:

\begin{enumerate}

	\item The developer does {\em not} have
	access to $D_{test}$.

	\item At the beginning, the developer specifies: (1) $T$, the number
	of development iterations; (2) $\epsilon$, the
	\emph{tolerance} of {\em distributional overfitting} (defined later); and (3) $1-\delta$,
	the \emph{confidence}.
	In response, the system returns the required size $|D_{test}|$.

	\item At every single step $t$, the system returns to the user
	an indicator $I_t \in \{1, \cdots, m\}$ which is a function of
	$D_{test}$. As we will see, $I_t=i$ indicates that degree of
	overfitting is bounded by $[\ubar{r}_i-\epsilon, \bar{r}_i+\epsilon]$, with
	probability at least $1-\delta$.

\end{enumerate}

\subsubsection{Formal Semantics of Overfitting}

There could be various definitions and semantics of overfitting.
It is not our goal to investigate all those alternatives in this work, which is
itself an interesting topic.
Instead, we settle on the following definitions that we believe are useful via
our conversations with ML application developers.

\vspace{0.5em}
\noindent
{\bf ``Distributional Overfitting'' and ``Empirical Overfitting''.}
Formally, let $H$ be an ML application and $l$ be a performance measure
(e.g., a loss function). For a given data set $D_{test}$ drawn i.i.d. from
$\mathcal{D}_{test}$, we use $l(H, D_{test})$ to represent the performance
of $H$ over $D_{test}$. We also use $\mathbb{E}_{D\sim\mathcal{D}_{test}}[l(H, D)]$
to represent the \emph{expected} performance of $H$ over the distribution
$\mathcal{D}_{test}$. We define the {\em degree of overfitting of the
validation set} as
$$
   \OVFT(D_{val}, \mathcal{D}_{test})
   = l(H, D_{val})  - \mathbb{E}_{D\sim\mathcal{D}_{test}}[l(H, D)].
$$
We decompose this term into
two terms -- the empirical
difference between $D_{val}$
and ${D}_{test}$, i.e.,
$$
    \Delta_{H}(D_{val}, D_{test}) =
    l(H, D_{val})  - l(H, D_{test})
$$
and the ``quality'' of estimator relying on ${D}_{test}$, i.e.,
$$
     \Delta_{H}(D_{test}, \mathcal{D}_{test}) :=
         l(H, D_{test}) - \mathbb{E}_{D\sim\mathcal{D}_{test}}[l(H, D)].
$$
We call the first term $\Delta_{H}(D_{val}, D_{test})$ {\em empirical
overfitting} as it is measured in terms of the empirical loss, and call
the second term $\Delta_{H}(D_{test}, \mathcal{D}_{test})$ {\em distributional
overfitting} as it measures the gap between the current test set and the
(unknown) true distribution. One crucial design decision we made in \meter is
to decouple these two terms and only report the empirical overfitting to the
user while treating distributional overfitting as a hard error tolerance
constraint.

\vspace{0.5em}
\noindent
{\bf Distributional Overfitting (Overfitting).} As the empirical overfitting
term can be measured directly by calculating the difference between the
validation set and the test set, the technical challenge of \meter~hinges
on the measurement/control of distributional overfitting. In the rest of this
paper, we use the term ``$D_{test}$ overfits by $\epsilon$'' to specifically
refer to distributional overfitting. When the context of $D_{test}$ and
$\mathcal{D}_{test}$ is clear, we use
\begin{align*}
   \Delta_H \equiv \Delta_{H}(D_{test}, \mathcal{D}_{test})
\end{align*}
to denote distributional overfitting.

We want to measure $\Delta_H $ not just for a single model $H$, but in the
context of a series $H_0,...,H_T$ of models.

\begin{definition}\label{definition:history}
We say that $D_{test}$ overfits by $\epsilon$ with respect to a submission
history $H_0, ..., H_T$ and performance measure $l$, if and only if
$\exists H \in \mathcal{H}^{(T)},~|\Delta_H| > \epsilon.$
Here $\mathcal{H}^{(T)}=\{H_0, ..., H_T\}.$
\end{definition}

Intuitively, this guarantees that, as long as the test set $D_{test}$ does not
overfit up to step $t$, all decisions made by the developer are according
to a test set that closely matches the real distribution -- at least in terms of
some aggregated statistics (e.g., accuracy).



	\begin{figure}[t]
	\centering
		\subfigure[System X (SemEval 2019)]{
	 		\includegraphics[width=0.22\textwidth]{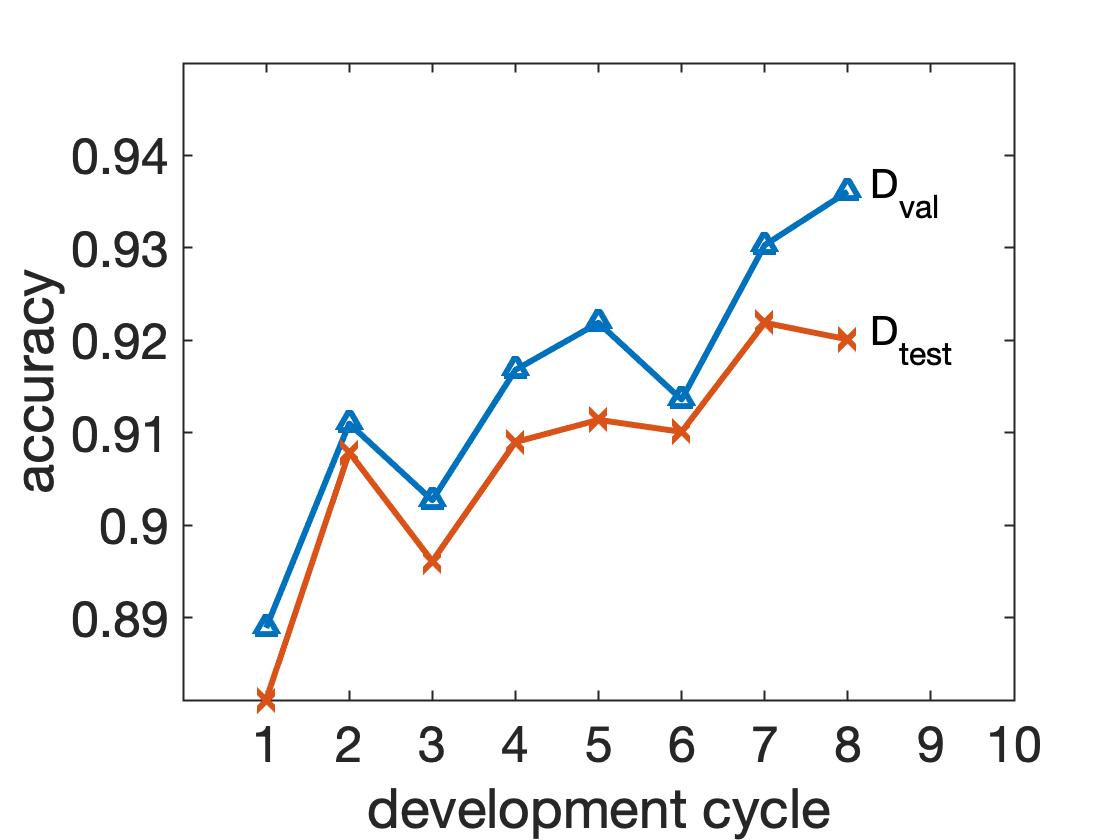}
	        \label{fig:development-history-system-x}
		}
		\subfigure[System Y (SemEval 2018)]{
	 		\includegraphics[width=0.22\textwidth]{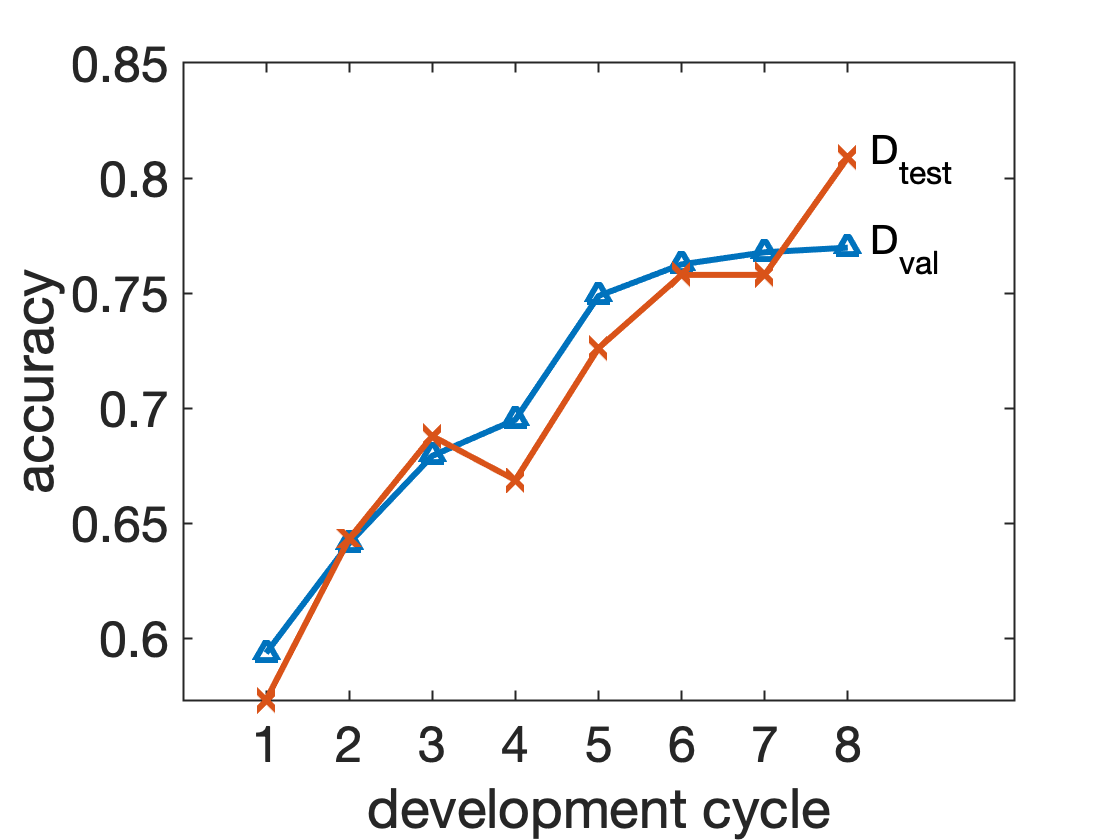}
	        \label{fig:development-history-system-y}
		}
		\subfigure[System Y (SemEval 2019) for $\delta = 0.1$ and $\epsilon = 0.01$]{
	 		\includegraphics[clip, trim=0cm 2.5cm 0cm 2.5cm, width=0.45\textwidth]{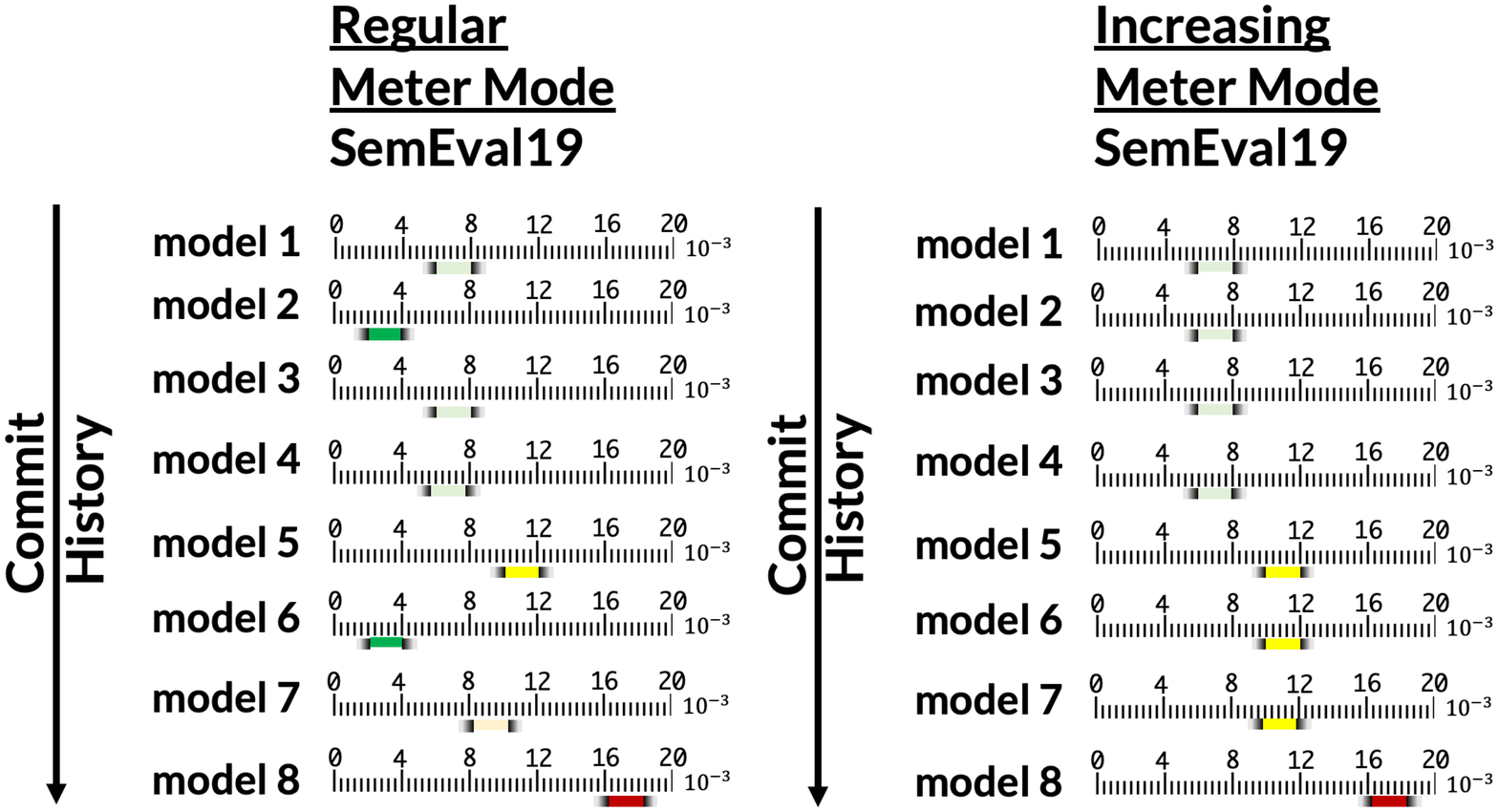}
	 		\label{fig:meters-for-y}
		}
	\caption{(a, b) Real development
			history traces of two ML applications; (c) Signals
			that developers would get when applying \meter
			to these two development traces.} 
	\label{fig:case-studies:meter}
	\end{figure}

\subsection{Example Use Cases} \label{sec:sys:use-cases}

To illustrate how \meter can be used in the development process of ML models, we
use the development trace data from two real-world ML applications we developed
in the past and showcase the signals \meter would return to its user.

\paragraph*{Development Trace 1: Emotion Detection in Text}
As our first case study, we took the development history of System X, which
is a participant of the ``\emph{EmoContext}'' task in SemEval
2019.\footnote{\url{https://www.humanizing-ai.com/emocontext.html}}
This task aims for detecting emotions from text leveraging contextual
information, which is deemed challenging due to the lack of facial expressions
and voice modulations.\footnote{\url{https://competitions.codalab.org/competitions/19790}}
It took developers eight iterations before delivering the final version of
System X. Changes in each individual step include adding various word
representations such as ELMo~\cite{ELMo} and GloVe~\cite{GloVe}, which lead to
significant performance increase/drop.
Figure~\ref{fig:development-history-system-x} plots the accuracy of System X on the validation set and the test set (assuming that the accuracy on the test set was reported to the user in every development step), respectively, in each of
the eight development steps.

\paragraph*{Development Trace 2: Relation Extraction}
Our second case study comes from System Y~\cite{RotsztejnHZ18}, which is a
participant of the task ``Semantic Relation Extraction and Classification in
Scientific Papers'' in SemEval
2018.\footnote{\url{https://competitions.codalab.org/competitions/17422}}
This task aims for identifying concepts from scientific documents and
recognizing the semantic relation that holds between the concepts.
In particular, it requires semantic relation extraction and classification
into six categories specific to scientific literature. The development history
of System Y indicates that it involves eight steps before reaching the final (version of the) system.
Figure~\ref{fig:development-history-system-y} presents the accuracy of
System Y on the training set (using 5-fold cross validation) and the test set,
respectively, for each step in the development cycle.

\paragraph*{Meter in Action}
Figure~\ref{fig:case-studies:meter} illustrates the signals that developers would
receive when applying \meter to these two development traces.
At each step, the current empirical overfitting is visible while the distributional overfitting is guaranteed to be smaller than $\epsilon = 0.01$ with probability $1 - \delta = 0.9$ (i.e., $\delta = 0.1$).

Figure~\ref{fig:case-studies:meter} also reveals two working modes of \meter: (1) the \emph{regular meter} and (2) the \emph{incremental meter}.
The regular meter simply returns an overfitting signal for each submission that indicates its degree of overfitting, as we have been discussing so far.
However, this is often unnecessary in practice, as developers usually only care about the \emph{maximum} (or \emph{worst}) degree of overfitting of \emph{all} the models that have been submitted.
The rationale is that the tolerance for overfitting usually only depends on the application, not a particular model --- a submitted model is acceptable as long as its overfitting is below the application-wide tolerance threshold.
The incremental meter is designed for this observation.

It is worth mentioning some tradeoffs if developers choose to use the incremental meter.
On the positive side, it can significantly reduce the number of human labels, compared with the regular meter (see Sections~\ref{sec:meter:regular} and~\ref{sec:meter:incremental}).
For instance, for the particular setting here ($\epsilon=0.01$ and $\delta=0.1$), the incremental meter would have required only 50K labels, compared with the 80K labels required by the regular meter, to support $T=8$ steps as in Figure~\ref{fig:meters-for-y}.
On the negative side, developers may lose clue on the performance of an individual submission, if the incremental meter does not march --- in such case developers only know that this submission is better than the \emph{worst} one in the history.

\subsection{Discussion}
One may wonder why not taking a more straightforward approach that bounds $|\OVFT(D_{val}, \mathcal{D}_{test})|$ directly, rather than the decomposition strategy \meter uses.
The rationale is that the former becomes much more challenging, if not impossible, when the validation distribution drifts from the true distribution (i.e., $\mathcal{D}_{val}\neq\mathcal{D}_{real}$).
When there is no distribution drift, we can indeed apply the same techniques in Section~\ref{sec:meter} below to $D_{val}$ (in lieu of $D_{real}$) to derive a lower bound for $|D_{val}|$.
However, given that the developer has full access to $D_{val}$ and $D_{val}$ is often used for hyper-parameter tuning that involves lots of iterations (i.e., very large $T$'s), the required $|D_{val}|$ can easily blow up.

In fact, it is even not our goal to bound $|\OVFT(D_{val}, \mathcal{D}_{test})|$.
Recall that, \meter~aims for understanding sizes of \emph{both} the validation set and the test set.
Even if directly bounding $|\OVFT(D_{val}, \mathcal{D}_{test})|$ were possible, it would only give us an answer to the question of desired validation set size, and the question about desired test set size remains unanswered.
Our decision of decomposing $\OVFT(D_{val}, \mathcal{D}_{test})$ is indeed a design choice, not a compromise.
Instead of providing a specific number about the validation set size, \meter~answers the question in probably the strongest sense: \emph{Pick whatever size --- the validation set size no longer matters!}
In practice, one can simply take a conservative, progressive approach: Start with a validation set with moderate size, and let the meter tell the degree of overfitting (via explicit control over the test set size); If the degree of overfitting exceeds the tolerance, replace the validation set (e.g., by adding more samples).
Therefore, our decomposition design is indeed a ``two birds, one stone'' approach that simultaneously addresses the two concerns regarding both validation and test set sizes.

Given that we do not explicitly bound $|\OVFT(D_{val}, \mathcal{D}_{test})|$, one may raise the question about the semantics of overfitting signals by \meter, in terms of $|\OVFT(D_{val}, \mathcal{D}_{test})|$.
When \meter~returns an overfitting signal $i$, it indicates the corresponding range $[\ubar{r}_i, \bar{r}_i]$ on the meter.
Since $\ubar{r}_i \leq |\Delta_{H}(D_{val}, D_{test})| \leq \bar{r}_i$ and $|\Delta_{H}(D_{test}, \mathcal{D}_{test})|\leq\epsilon$,
it follows that, with probability (i.e., confidence) at least $1-\delta$,
$$\ubar{r}_i-\epsilon\leq |\OVFT(D_{val}, \mathcal{D}_{test})|\leq \bar{r}_i+\epsilon.$$

\paragraph*{Limitations and Assumptions}
Although the above setup is quite generic, a range of limitations remain.
One important limitation is that \meter assumes that {\em user
is able to draw samples from each distribution at any time/step}. This is not always true, especially in many
medical-related applications and anomaly detection
applications in physical systems. Another limitation is
that \meter assumes that {\em each data distribution
is stationary}, i.e., all three distributions
do not change or drift over time, though in many applications
concept/domain drift is inevitable~\cite{GamaMCR04, Zliobaite10}.
We believe that these limitations are all interesting directions to explore.
However, as one of the early efforts on overfitting management, we leave these as
future work in this paper.

\section{Monitoring Overfitting}\label{sec:meter}

We now present techniques in \meter~that monitor the degree of (distributional) overfitting.
We first piggyback on techniques recently developed in adaptive analysis~\cite{blum2015ladder,DworkFHPRR15} and apply them to our new scenario.
We have also developed multiple simple, but novel, optimizations, which we will
discuss later in Section~\ref{sec:optimization}.
The main technical question we aim to answer is that, given the error tolerance
parameters $\epsilon$, $\delta$, and the length of development
cycles $T$, \emph{how large should the test set $D_{test}$ be}?

\begin{figure}[t]
\centering
\includegraphics[width=0.45\textwidth]{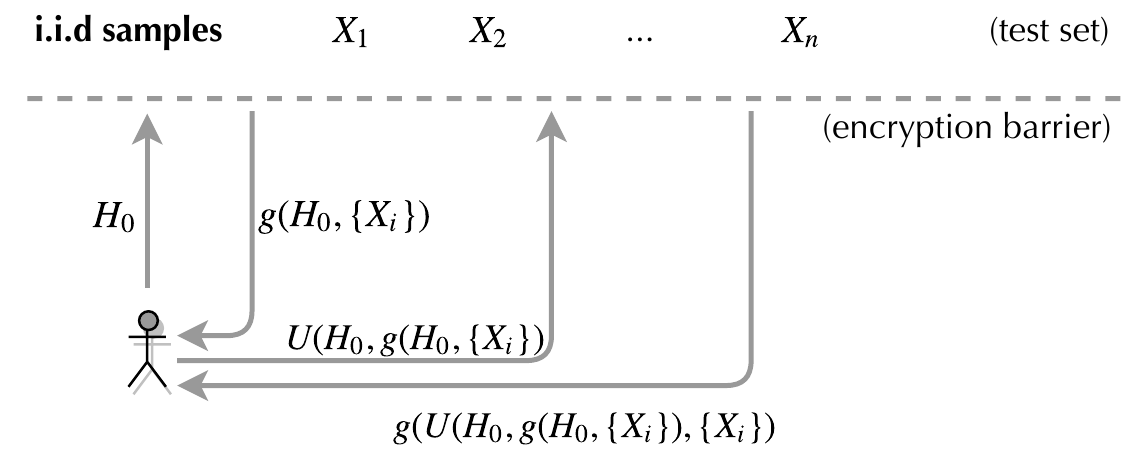}
\vspace{1em}
\caption{An illustration of adaptive analysis.
		The performance of the
		$t^{\text{th}}$ model submission $H_{t}(\{X_i\})$ on the test
		set $\{X_i\}$ is reported back to the
		developer indirectly via a feedback function $g(\cdot)$.}
\label{fig:adaptive}
\end{figure}

\subsection{Recap: Adaptive Analysis}\label{sec:meter:adaptive-analysis}

As we have discussed in Section~\ref{sec:preliminaries:adaptive}, we cannot simply draw a fresh test set for each new submission (i.e., the \textbf{Resampling} baseline), as it would become prohibitively expensive in practice for reasonable choices of $\epsilon$ and $\delta$ in many circumstances.
For instance, if we set $\epsilon=0.01$ and $\delta=0.01$, it would require 380K examples to test just $T=10$ models (by Equation~\ref{eq:resampling}).

To reduce this sample complexity, it is natural to consider \emph{reusing} the same test set for subsequent submissions.
As one special case (which is
unrealistic for \meter), if all submissions are \emph{independent} (i.e., the next submission does not depend on the overfitting signal returned by \meter~for the present submission), then we can simply apply the union bound combined with the Hoeffding's inequality to conclude a sample complexity as shown in Equation~\ref{eq:independent}.
Using the previous setting again ($\epsilon=0.01$, $\delta=0.01$, and $T=10$), we now need only 38K examples in the test set.

However, this independence assumption seldom holds in practice, as the developers would always
receive overfitting signals returned by \meter,
which in the worst case, would always
have impact on her choice of the next model
(see Figure~\ref{fig:adaptive}).
It then implies that the models submitted by developers can be \emph{dependent}.
We now formally examine this kind of \emph{adaptive analysis} during ML application development in more detail, studying its impact on the size of the test set.


\subsection{Cracking Model Dependency}\label{sec:meter:description-length}

The basic technique remains similar to when the submissions are independent: We can (1) apply Hoeffding's inequality to each submission and then (2) apply union bound to \emph{all possible submissions}.
While (1) is the same as
the independent case, (2) requires additional work as the set of all possible submissions expands significantly under adaptive analysis.
We use a technique based on \emph{description length}, which is similar to those used by other adaptive analysis work~\cite{blum2015ladder, DworkFHPRR15, CI}.

Specifically, consider step $t$.
If the submission $f_{t+1}$ is independent of $f_t$, the number of all possible submissions is simply $T$ after $T$ steps.
If, on the other hand, $f_{t+1}$ depends on both $f_t$ and $I_t$ (i.e., the indicator returned by the meter, which specifies the range $i\in[m]$ that the degree of overfitting of $f_t$ falls into), then for each different value of $I_t$ we can have a different $f_{t+1}$.
To count the total number of possible submissions, we can naturally use a \emph{tree} $\mathcal{T}(m,T)$ to capture dependencies between $f_{t+1}$ and $f_t$.

In more detail, the tree $\mathcal{T}(m,T)$ contains $T+1$ levels, where level $t$ represents the corresponding step $H_t$ in the submission history $\mathcal{H}^{(T)}=\{H_0,...,H_{T}\}$.
In particular, the root represents $H_0$.
Each node at level $t$ represents a particular realization $f_t$ of $H_t$, i.e., a possible submission made by developers at step $t$.
Meanwhile, the children of $f_t$ represent all possible $f_{t+1}$'s that are realizations of $H_{t+1}$ \emph{given that the submission at step $t$ is $f_t$}.

\begin{example}[Dependency Tree]
Figure~\ref{fig:regular-meter} showcases the corresponding tree $\mathcal{T}(2,3)$ for a regular meter when $m=2$ and $T=2$.
It contains $T+1=3$ levels.
The root represents $H_0$, the initial submission.
Since $m=2$, the meter contains two overfitting ranges and therefore can return one of the two possible overfitting signals, Signal 1 or Signal 2.
Depending on which signal is returned for $H_0$, developers may come up with different realizations for $H_1$.
This is why the root has two children at level $1$.
The same reasoning applies to these two nodes at level $1$ as well, which results in four nodes at level $2$.
\end{example}

The problem of applying union bound to all possible submissions under adaptive analysis therefore boils down to computing the size $|\mathcal{T}(m,T)|$ of the tree $\mathcal{T}(m,T)$.
We next analyze $|\mathcal{T}(m,T)|$ for the regular meter and the incremental meter, respectively.

\subsection{Regular Meter}\label{sec:meter:regular}



In the regular meter, each signal
can take $m$ values ($\{1...m\}$)
for a meter with $m$ possible signals.
One can then easily see that, in general, the number of nodes in the (model) dependency tree $\mathcal{T}(m,T)$ is
$|\mathcal{T}(m,T)|=\sum_{t=1}^T m^t=\frac{m(m^T-1)}{m-1}$.
This leads to the following result on sample complexity for the regular meter (the complete proof is in Appendix~\ref{appendix:theory:proofs:uniform}):



\begin{theorem}[Regular Meter]\label{theorem:uniform}
The test set size (in the adaptive setting) of the regular meter satisfies
$$ 2\cdot |\mathcal{T}(m,T)|\cdot\exp \left(-2|D_{test}| \epsilon^2 \right)<\delta,$$
where $|\mathcal{T}(m,T)|=\frac{m(m^T-1)}{m-1}$. As a result, it follows that
\begin{equation}\label{eq:uniform:regular}
    |D_{test}| > \frac{\ln\big(2|\mathcal{T}(m,T)|/\delta\big)}{2\epsilon^2}\approx\frac{T\ln m + \ln \big(2m/((m-1)\delta)\big)}{2\epsilon^2},
\end{equation}
by using the approximation $m^T-1\approx m^T$.
 \end{theorem}

\noindent
\textbf{(Comparison to Baseline)}
Compared to the \textbf{Resampling} baseline (Equation~\ref{eq:resampling}), the regular meter (with $m=5$) can reduce the number of test examples from 380K to 108K when $\epsilon=0.01$, $\delta=0.01$, and $T=10$, a 3.5$\times$ improvement.




	\begin{figure}
	\centering
		\subfigure[Regular Meter]{
		    \includegraphics[width=0.2\textwidth]{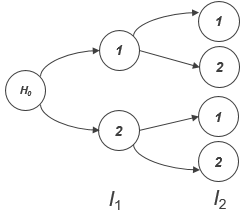}
			\label{fig:regular-meter}
		}
		\subfigure[Incremental Meter]{
		    \includegraphics[width=0.2\textwidth]{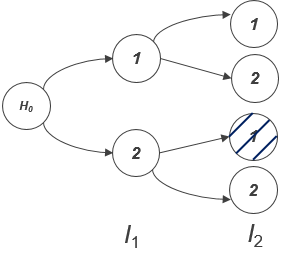}
			\label{fig:incremental-meter}
		}
	\caption{Illustration of regular and incremental meters.}
	\label{fig:proof-idea}
	\end{figure}

\subsection{Incremental Meter}\label{sec:meter:incremental}

As we have discussed in Section~\ref{sec:sys:use-cases}, the incremental meter reports the worst degree of overfitting for all models that have been submitted so far. Formally, at step $t$ it returns
\begin{equation*}
    \mathbb{I}^{(t)}=\max\{I_k|1\leq k\leq t, I_k\in\{1,\cdots,m\}\},
\end{equation*}
where $I_k$ is the overfitting signal that would have been returned by the regular meter at step $k$.
As a result,
the incremental meter can only move (indeed, increase) towards
one direction.
This constraint allows us to further reduce the required amount of test examples,
often significantly (compared to the regular meter):

\begin{theorem} [Incremental Meter]\label{theorem:uniform-incremental}
The test set size (in the adaptive setting) of the incremental meter satisfies
\begin{equation*}
2\cdot |\mathcal{T}(m,T)|\cdot\exp \left(  -2|D_{test}| \epsilon^2 \right)<\delta \, ,
\end{equation*}
where
\begin{align*}
    |\mathcal{T}(m,T)| = \sum \limits_{k}^{m}\sum \limits_{t}^{T} {k+t-2 \choose t-1}  = {m + T \choose m } - 1.
\end{align*}
As a result, it follows that
\begin{equation}\label{eq:uniform:incremental}
   |D_{test}| > \frac{\ln(2 |\mathcal{T}(m,T)|/\delta)}{2\epsilon^2}=\frac{\ln(2\cdot\Big({m + T \choose m } - 1\Big)/\delta)}{2\epsilon^2}.
\end{equation}
\end{theorem}
The proof is in Appendix~\ref{appendix:theory:proofs:uniform-incremental}.
\footnote{The proof is quite engaged but the
idea is simple: Count the number of tree nodes $h(k, t)$ with respect to (1) the overfitting signal $k$ returned by \meter and (2) the level $t$, and observe that $|\mathcal{T}(m,T)|=\sum\nolimits_{k=1}^{m}\sum\nolimits_{t=1}^{T} h(k, t).$ We can further show that $h(k,t)={k+t-2 \choose t-1}$.}
Compared to the regular meter, the size of the (model) dependency tree can be further pruned. Figure~\ref{fig:incremental-meter} illustrates this for the incremental meter when $m=2$ and $T=2$: Shadowed nodes are pruned with respect to the tree of the corresponding regular meter (Figure~\ref{fig:regular-meter}).

\vspace{0.5em}
\noindent
\textbf{(Comparison to Baseline)}
Compared to the \textbf{Resampling} baseline (Equation~\ref{eq:resampling}), the incremental meter (with $m=5$) can reduce the number of test examples from 380K to 66K ($\epsilon=0.01$, $\delta=0.01$, and $T=10$), a 5.8$\times$ improvement.

\section{Optimizations}\label{sec:optimization}

In the previous section, we adapt existing
techniques directly to \meter. However, these
techniques are developed for general
adaptive analysis without considering the specific
application scenario that \meter~is designed for.
We now describe a set of simple, but
novel optimizations that further
decrease the requirement of labels by \meter.

\subsection{Nonuniform Error Tolerance}\label{sec:optimization:non-uniform}

Our first observation is that a uniform
error tolerance $\epsilon$ for all signals, as was assumed
by all previous work~\cite{blum2015ladder, CI},
is perhaps an ``overkill'' -- when empirical overfitting is large, user might have
higher error tolerance (e.g.,
when validation accuracy and test accuracy are
already off by 20 points, user might not need to control
distributional overfitting to
a single precision point).
Our first optimization is then to extend
existing result to support
different $\epsilon_k$'s for different
signals $k\in[m]$, such that $\epsilon_1\leq\cdots\leq\epsilon_m$. This leads to the following
extensions of the results in Sections~\ref{sec:meter:regular} and~\ref{sec:meter:incremental}.


\begin{corollary}[Nonuniform, Regular Meter]\label{corollary:nonuniform-regular}
The test set size (in the adaptive setting) of the regular meter with non-uniform $\epsilon$ satisfies
\begin{equation}\label{eq:nonuniform}
    \frac{1}{m}\sum\nolimits_{k=1}^{m} 2\cdot |\mathcal{T}(m,T)|\cdot\exp \left( -2|D_{test}| \epsilon_k^2 \right)<\delta,
\end{equation}
where $|\mathcal{T}(m,T)|=\frac{m(m^T-1)}{m-1}$ remains the same as in Theorem~\ref{theorem:uniform}.
\end{corollary}

The proof can be found in Appendix~\ref{appendix:theory:proofs:nonuniform}.
The basic idea is the same as that in Section~\ref{sec:meter:description-length}: We can use a tree $\mathcal{T}(m,T)$ to capture the dependencies between historical submissions and count the number of tree nodes (i.e., possible submissions).
However, in the nonuniform-$\epsilon$ case, we have to count the number of tree nodes for each individual overfitting signal $k\in[m]$ separately, since we need to apply the Hoeffding's inequality for each group of nodes $\mathcal{T}^{(k)}(m,T)$ corresponding to a particular $k$, with respect to $\epsilon_k$.
For the regular meter, it turns out that $|\mathcal{T}^{(k)}(m,T)|=\frac{1}{m}|\mathcal{T}(m,T)|$.

\begin{corollary}[Nonuniform, Incremental Meter]\label{corollary:nonuniform-incremental}
The test set size (in the adaptive setting) of the incremental meter with non-uniform $\epsilon$ satisfies
\begin{equation}\label{eq:nonuniform-incremental}
    \sum\nolimits_{k=1}^{m} 2\cdot | \mathcal{T}^{(k)} (m,T) | \cdot \exp\left(-2|D_{test}|\epsilon_k^2\right) <\delta,
\end{equation}
where
    \begin{align*}
        |\mathcal{T}^{(k)}(m,T)| = \sum \limits_{t =1}^{T}  { k + t -2\choose t-1} = {k + T -1 \choose k}.
    \end{align*}
\end{corollary}

The previous remark on the proof of Corollary~\ref{corollary:nonuniform-regular} can be applied to the nonuniform, incremental meter, too: We can compute $|\mathcal{T}^{(k)} (m,T)|$ and then apply the Hoeffding's inequality with respect to $\epsilon_k$, for each $k\in[m]$ separately.
The complete proof can be found in Appendix~\ref{appendix:theory:proofs:nonuniform-incremental}.
Note that the tree size $|\mathcal{T}(m,T) | = \sum \nolimits_{k=1}^{m} | \mathcal{T}^{(k)}(m,T) |$ is the same as that of Theorem~\ref{theorem:uniform-incremental}.



\paragraph*{Impact on Sample Complexity}
The difficulty of finding a closed form
solution for non-uniform $\epsilon$ makes
it challenging to directly compare this optimization
with those that we derived in Section~\ref{sec:meter}.
To better understand sample complexity of nonuniform meters, in the following we conduct an analysis based on the assumption that $|D_{test}|$ is dominated by $\epsilon_1$:\footnote{Given that $\epsilon_1\leq\cdots\leq\epsilon_m$ and the exponential terms that enclose the $\epsilon_k$'s, one can expect that the LHS sides of Equations~\ref{eq:nonuniform} and~\ref{eq:nonuniform-incremental} are dominated by the terms that contain $\epsilon_1$.}

\vspace{0.5em}
\noindent
\textbf{(Sample Complexity assuming $\epsilon_1$-Dominance)}
Specifically, for the nonuniform, regular meter, we have
\begin{equation}\label{eq:nonuniform:approx}
    |D_{test}| > \frac{1}{2\epsilon_1^2}\Big(\ln\frac{2(m^T-1)}{\delta(m-1)}\Big)\approx\frac{T\ln m +\ln\frac{2}{\delta(m-1)}}{2\epsilon_1^2}.
\end{equation}
On the other hand, for the nonuniform, incremental meter, we have
\begin{equation}\label{eq:nonuniform-incremental:approx}
    |D_{test}| > \frac{1}{2\epsilon_1^2}\Big(\ln\frac{2T}{\delta}\Big).
\end{equation}


\noindent
\textbf{(Improvement over Section~\ref{sec:meter})}
We now compare the sample complexity of the nonuniform meters to their uniform counterparts.
The nonuniform, regular meter can reduce sample complexity to $100K$ when setting $(\epsilon_1,\epsilon_2,\epsilon_3,\epsilon_4,\epsilon_5)$ to (0.01, 0.02, 0.03, 0.04, 0.05), $\delta=0.01$, and $T=10$, compared to $108K$ with a uniform error tolerance $\epsilon = 0.01 \equiv \epsilon_1$.
We do not observe significant improvement, though.
In fact, if we compare Equation~\ref{eq:nonuniform:approx} with Equation~\ref{eq:uniform:regular}, the improvement is upper bounded by $1+\frac{\ln m}{\ln(2/\delta)}$. For $m=5$ and $\delta=0.01$, it means that the best improvement would be $1.3\times$ regardless of $T$.
Nonetheless, one can increase either $m$ or $\delta$ to boost the expected improvement.
On the other hand, the nonuniform, incremental meter can further reduce sample complexity from $66K$ to $38K$ (which matches Equation~\ref{eq:independent}, the ideal sample complexity when all submissions are independent), a $1.75\times$ improvement.

\subsection{Multitenancy}\label{sec:optimization:multitenancy}

Our second observation is that, when multiple users are having access to the same meter, it is possible to
decrease the requirement on the number of examples if {\em we assume that these users do not communicate with each other.}
Multitenancy is a natural requirement in practice given that developing ML applications is usually team work.
We implemented a multitenancy management subsystem to enable concurrent access to the meter from different users.

\begin{figure}
\centering
\includegraphics[width=0.45\textwidth]{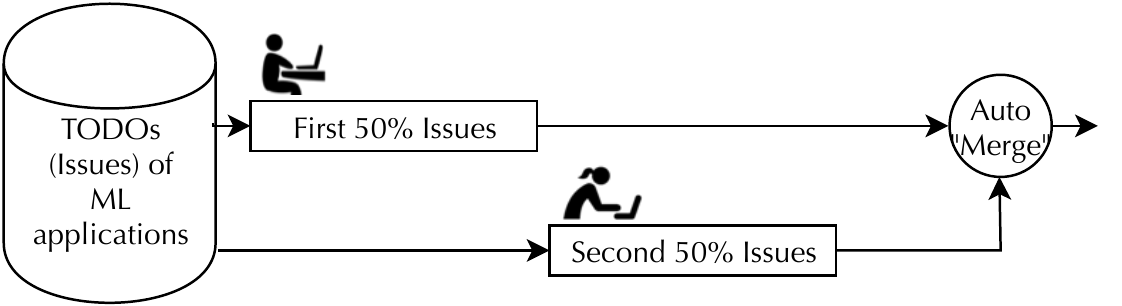}
\caption{Multitenancy in \meter.}
\label{fig:multitenancy}
\end{figure}

Figure~\ref{fig:multitenancy} illustrates the specific multitenancy scenario targeted by \meter.
Unlike traditional multitenancy setting where multiple users access system simultaneously, our multitenancy scenario captures more of a collaboration pattern between two developers where one starts from the checkpoint made by the other, similar to the well-known ``\texttt{git branch}'' and ``\texttt{git merge}'' development pattern (if we have to make an analogy).

An interesting, perhaps counter-intuitive observation is that multitenancy can
further \emph{reduce} the desired size of the test set. We illustrate this using
a simple, two-tenant case where there are only two developers, each working on
$\frac{T}{2}$ steps. We provide the result for more than two developers in
Appendix~\ref{appendix:multitenant-meter}.
We focus on nonuniform meters in our discussion, as uniform meters (i.e., with a single $\epsilon$) can be viewed as special cases.

\vspace{0.5em}
\noindent
\textbf{(Two Tenants in Nonuniform, Regular Meter)}
Consider the regular meter first. Since each tenant has $\frac{T}{2}$ development cycles, it follows from
Corollary~\ref{corollary:nonuniform-regular} that, in the presence of two tenants,
\begin{align}\label{eq:bitenant-multistep:regular}
\delta &>  \frac{1}{m}\sum\nolimits_{k=1}^{m} 4 \cdot  | \mathcal{T}^{k}(m,T/2)|\cdot \exp(-2 | D_{test}| \epsilon_k^2) \\\nonumber
&=  \sum\nolimits_{k=1}^{m} 4 \cdot\frac{m^{T/2}-1}{m-1} \cdot \exp(-2 | D_{test}| \epsilon_k^2)\,.
\end{align}

\noindent
\textbf{(Two Tenants in Nonuniform, Incremental Meter)}
Similarly, for the incremental meter, it follows from
Corollary~\ref{corollary:nonuniform-incremental} that, in the presence of two tenants,
\begin{align}\label{eq:bitenant-multistep:incremental}
\delta &>  \frac{1}{m}\sum\nolimits_{k=1}^{m} 4 \cdot  | \mathcal{T}^{k}(m,T/2)|\cdot \exp(-2 | D_{test}| \epsilon_k^2) \\ \nonumber
&=  \sum\nolimits_{k=1}^{m} 4 \cdot\binom{k + \frac{T}{2}-1}{k} \cdot \exp(-2 | D_{test}| \epsilon_k^2) \, .
\end{align}


\paragraph*{Intuition}
This result may be counter-intuitive at a first glance.
One might wonder what the fundamental difference is between multiple
developers and a single developer, given our two-tenancy setting.
The observation here is that the second developer can \emph{forget}
about the ``development history'' made by the first developer, since
it is irrelevant to her own development (see Figure~\ref{fig:multitenancy-example} for an example when $m=2$ and $T=3$).

\paragraph*{Impact on Sample Complexity}
We can illustrate the impact of multi-tenancy
by again analyzing sample complexity under the $\epsilon_1$-dominance assumption:

\vspace{0.5em}
\noindent
\textbf{(Sample Complexity Assuming $\epsilon_1$-Dominance)}
Specifically, for the nonuniform, regular meter, we have
\begin{equation}\label{eq:two-tenancy:regular:approx}
    |D_{test}| > \frac{1}{2\epsilon_1^2}\Big(\ln\frac{4(m^{T/2}-1)}{\delta(m-1)}\Big)\approx\frac{\frac{T}{2}\ln m +\ln\frac{4}{\delta(m-1)}}{2\epsilon_1^2}.
\end{equation}
On the other hand, for the nonuniform, incremental meter, we have
\begin{equation}\label{eq:two-tenancy:incremental:approx}
    |D_{test}| > \frac{1}{2\epsilon_1^2}\Big(\ln\frac{2T}{\delta}\Big).
\end{equation}

\noindent
\textbf{(Improvement over Section~\ref{sec:optimization:non-uniform})}
We now compare the sample complexity of the two-tenancy meters to the single-tenancy ones.
The two-tenancy, regular meter can reduce the number of test examples from $100K$ to $71K$ when setting $(\epsilon_1,...,\epsilon_5)$ = (0.01, 0.02, 0.03, 0.04, 0.05), $\delta=0.01$, and $T=10$, a $1.4\times$ improvement.
One can indeed achieve asymptotically $2\times$ improvement as $T$ increases (see Section~\ref{sec:evaluation:optimizations:multitenancy}).
On the other hand, the two-tenancy, incremental meter cannot further improve the sample complexity.
This is not surprising, as Equation~\ref{eq:two-tenancy:incremental:approx} and Equation~\ref{eq:nonuniform-incremental:approx} are exactly the same.
In fact, both have matched Equation~\ref{eq:independent}, which represents the ideal sample complexity when all submissions are independent.

\begin{figure}[t]
\centering
	\subfigure[Single Tenant]{%
	    \includegraphics[width=0.2\textwidth]{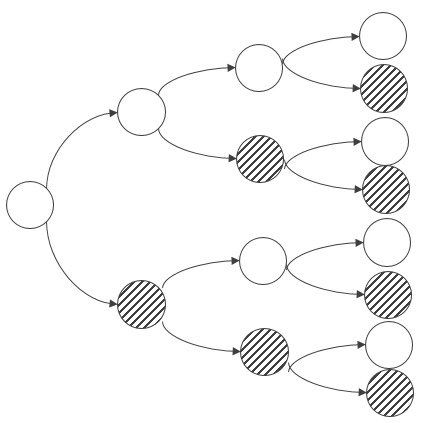}%
	    \label{fig:single-developer}%
	    }
	\subfigure[Two Tenants]{%
        \includegraphics[width=0.15\textwidth]{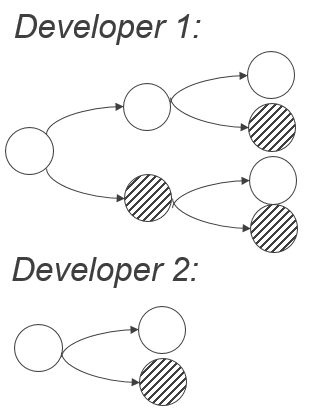}%
		\label{fig:two-developers}%
	}
\caption{Comparison between single-/multi- tenancy.}
\label{fig:multitenancy-example}
\end{figure}

\subsection{``Time Travel''}\label{sec:optimization:time-travel}
Our third observation is that it is a natural action for developers to \emph{revert} to a previous step once an overfitting signal is observed, i.e., ``traveling back in timeline.''
In practice, it makes little sense to revert to older steps except for the latest one prior to the one that resulted in overfitting.
Therefore, we only consider the case of taking one step back, just like what the ``\texttt{git revert HEAD}'' command does.
Intuitively, ``time travel'' permits ``regret'' in development history.
As a result, we can use a smaller test set to support the same number of development steps.
In the following, we present a formal analysis that verifies this intuition.


\paragraph*{Analysis}
For ease of exposition, we start by assuming one single budget for ``time travel,'' i.e., only one reversion is allowed in development history.
Again, we are interested in the total number of possible model submissions.
Suppose that user decides to revert at step $t$.
One can decompose the entire procedure into three phases:
(1) keep submitting models until step $t$; (2) revert and go back to step $t-1$;
(3) continue submitting $T-t$ times until step $T$.
For each $\epsilon_k$, the number of possible submissions in the three phases are then (1) $|\mathcal{T}^{(k)}(m,t)|$, (2) $-|\mathcal{T}^{(k)}(m,t-1)|$, and (3) $|\mathcal{T}^{(k)}(m,(t-1)+(T-t))| = |\mathcal{T}^{(k)}(m,T-1)|$.


It is straightforward to generalize the results to $B\geq 1$ budgets.
Suppose that user decides to revert at steps $t_1\leq t_2\leq\cdots\leq t_B$.
Both phases (1) and (2) can repeat for each $t_i$, though each $t_i$ should be replaced by $t'_i=t_i-(i-1)$ to accommodate the ``time shift'' effect due to ``time travel.'' Phase (3) follows afterwards.
Therefore, the total number of possible submissions is
\begin{equation*}
    \sum\nolimits_{i=1}^{B}\Big(|T^{(k)}(m, t'_i)| - |T^{(k)}(m, t'_i-1)|\Big) + |T^{(k)}(m, T-B)|.
\end{equation*}

\noindent
\textbf{(``Time Travel'' in Nonuniform, Regular Meter)}
By Corollary~\ref{corollary:nonuniform-regular}, the total number of possible submissions is
\begin{equation}\label{eq:time-travel:regular}
|\mathcal{R}^{(k)}|= \frac{(m^{T-B}-1)}{m-1} + \sum\nolimits_{i=1}^{B}m^{t'_i-1}.
\end{equation}
The test set size therefore satisfies
\begin{equation}\label{eq:time-travel:regular:sample-size}
    \sum\nolimits_{k=1}^{m} 2\cdot |\mathcal{R}^{(k)}| \cdot \exp\left(-2|D_{test}|\epsilon_k^2\right) <\delta.
\end{equation}

\vspace{0.5em}
\noindent
\textbf{(``Time Travel'' in Nonuniform, Incremental Meter)}
By Corollary~\ref{corollary:nonuniform-incremental}, the total number of possible submissions is
\begin{equation}\label{eq:time-travel:incremental}
|\mathcal{I}^{(k)}|= { k+(T-B)-1 \choose k } + \sum\nolimits_{i=1}^{B}{ k+t'_i-2 \choose k-1 }.
\end{equation}
The test set size therefore satisfies
\begin{equation}\label{eq:time-travel:incremental:sample-size}
    \sum\nolimits_{k=1}^{m} 2\cdot |\mathcal{I}^{(k)}| \cdot \exp\left(-2|D_{test}|\epsilon_k^2\right) <\delta.
\end{equation}

\paragraph*{Impact on Sample Complexity}
As before, we illustrate the impact of ``time travel'' by analyzing sample complexity under the $\epsilon_1$-dominance assumption.

\vspace{0.5em}
\noindent
\textbf{(Sample Complexity Assuming $\epsilon_1$-Dominance)}
Specifically, for the nonuniform, regular meter, we have
\begin{equation}\label{eq:time-travel:regular:approx}
    |D_{test}| > \frac{1}{2\epsilon_1^2}\ln\Big(\frac{2}{\delta}\cdot\Big(\frac{m^{T-B}-1}{m-1} + \sum\nolimits_{i=1}^{B}m^{t'_i-1}\Big)\Big).
\end{equation}
On the other hand, for the nonuniform, incremental meter, \emph{regardless of} the choice of $B$ and $t_1$ to $t_B$, we have
\begin{equation}\label{eq:time-travel:incremental:approx}
    |D_{test}| > \frac{1}{2\epsilon_1^2}\ln\Big(\frac{2T}{\delta}\Big).
\end{equation}

\noindent
\textbf{(Improvement over Section~\ref{sec:optimization:non-uniform})}
We now compare the sample complexity of the ``time travel'' meters to the ones in Section~\ref{sec:optimization:non-uniform}.
We set $B=3$ and $(t_1,t_2,t_3)$ = (1, 2, 3).
The ``time travel'' regular meter can reduce sample complexity from $100K$ to $76K$ when setting $(\epsilon_1,...,\epsilon_5)$ = (0.01, 0.02, 0.03, 0.04, 0.05), $\delta=0.01$, and $T=10$, a $1.32\times$ improvement.
We can further improve the sample complexity by increasing the ``time travel'' budget --- for $T=10$ we can achieve an improvement up to $2.64\times$ (see Section~\ref{sec:evaluation:optimizations:time-travel}).
We do not see improvement for the incremental meters under the $\epsilon_1$-dominance assumption, though, as Equation~\ref{eq:time-travel:incremental:approx} already matches the lower bound in Equation~\ref{eq:independent}.


\section{Experimental Evaluation}
\label{sec:evaluation}


We report experimental evaluation results in this section.
Our evaluation covers the following aspects of \meter:
\begin{itemize}
    \item How effective are the regular meter and incremental meter in Section~\ref{sec:meter}, compared with the \textbf{Resampling} baseline in Section~\ref{sec:preliminaries:adaptive}? --- see Section~\ref{sec:evaluation:meter}.
    \item How effective are the optimization techniques in Section~\ref{sec:optimization}, compared with the basic regular meter and incremental meter in Section~\ref{sec:meter}? --- see Section~\ref{sec:evaluation:optimizations}.
    \item How can \meter~be fit into real ML application development lifecycle management? --- see Section~\ref{sec:evaluation:case-studies}.
\end{itemize}
We compare the participated techniques in terms of their induced sample complexity (i.e., the amount of human labels required), under various parameter settings that are typical in practice.





\subsection{Meters vs. Baselines}\label{sec:evaluation:meter}


\paragraph*{Baseline}
We have presented the \textbf{Resampling} baseline in Section~\ref{sec:preliminaries:adaptive} in the context of a single $\epsilon$, which can only serve as a baseline for uniform meters.
It is straightforward to extend it to the nonuniform case, though: Given that $\epsilon_1 \leq \cdots \leq\epsilon_m$, the required sample size in each single step is bounded by $\ln(2T/\delta)/(2\epsilon_1^2).$
Hence, the total number of samples $n_{tot}$ satisfies
\begin{equation}\label{eq:naive-multistep}
     | D_{test,tot}|  > T\cdot\frac{\ln(2T/\delta)}{2\epsilon_1^2}.
\end{equation}

\paragraph*{Computation of Sample Size for Nonuniform Meters}
Equations~\ref{eq:nonuniform} and~\ref{eq:nonuniform-incremental} are algebraically unsolvable for the sample size $n= |D_{test}|$ unless $m = 1$.
It is, however, possible to find the correct $n = |D_{test}|$ by bounding $1 < n  < \infty$ and using binary search to find the desired $\delta$ and $n$.



\paragraph*{Parameter Settings}
In our experiments, we choose $\delta$ from the set $\{0.001, 0.005, 0.01, 0.05, 0.1\}$, which consists of common confidence thresholds encountered in practice.
We choose $m$ from $\{5, 10, 20, 50\}$ and vary $T$ from 10 to 100.

\begin{figure}
\centering
	\subfigure[ $\delta=0.05$ and $\epsilon=0.05$]{
        \includegraphics[clip, trim=1cm 2cm 1cm 1cm, width=0.44\textwidth]{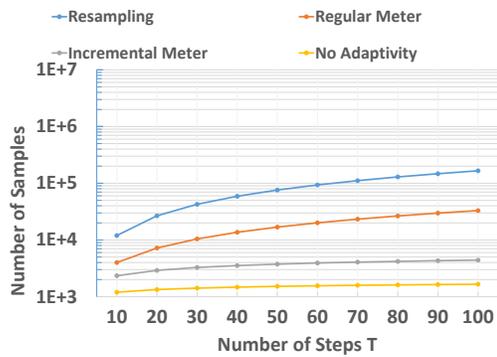}
        \label{fig:uniform-delta5-epsilon5}
	}
	\subfigure[$\delta=0.01$ and $\epsilon=0.05$]{
        \includegraphics[clip, trim=1cm 2cm 1cm 1cm, width=0.44\textwidth]{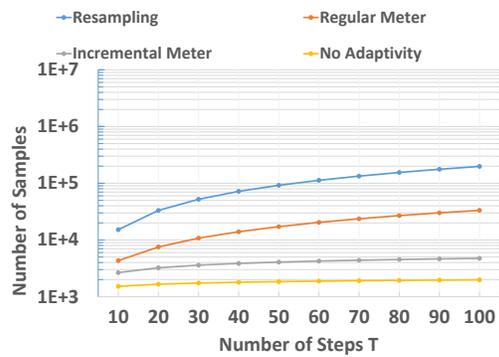}
        \label{fig:uniform-delta1-epsilon5}
	}
	\subfigure[$\delta=0.05$ and $\epsilon=0.01$]{
        \includegraphics[clip, trim=1cm 2cm 1cm 1cm, width=0.44\textwidth]{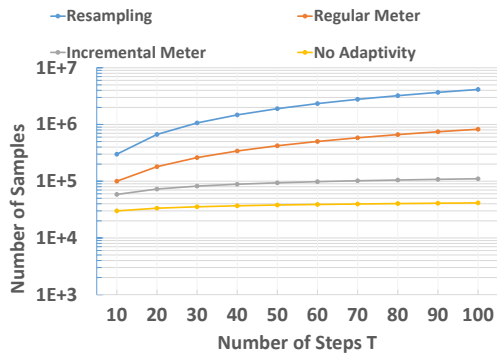}
        \label{fig:uniform-delta5-epsilon1}
	}
	\subfigure[$\delta=0.01$ and $\epsilon=0.01$]{
        \includegraphics[clip, trim=1cm 2cm 1cm 1cm, width=0.44\textwidth]{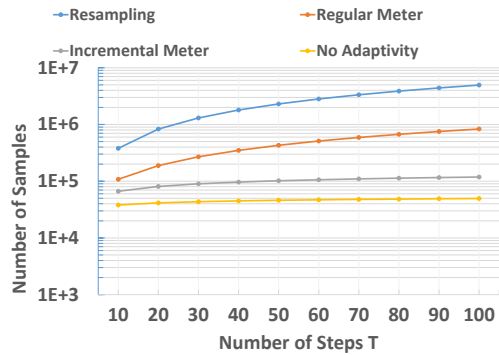}
        \label{fig:uniform-delta1-epsilon1}
	}
\caption{Comparison of uniform meters with baseline approaches when setting $m=5$ and varying $T$ from 10 to 100.}
\label{fig:uniform-meters}
\end{figure}

\paragraph*{Evaluation Results}
We choose $\epsilon$ from $\{0.01, 0.05, 0.1\}$. Figure~\ref{fig:uniform-meters} compares the uniform meters with the baseline approaches when setting $m=5$, $\delta=\{0.01, 0.05\}$, and $\epsilon=\{0.01, 0.05\}$.
In each of the four subfigures, we plot the number of samples required by the baseline (``\textbf{Resampling}''), the regular meter (``\textbf{Regular Meter}''), the incremental meter (``\textbf{Incremental Meter}''), and the \emph{ideal} case where the submissions are independent (``\textbf{No Adaptivity}''), respectively. Note that the $y$-axis is in log scale.

\vspace{0.5em}
\noindent
\textbf{(Impact of $T$)}
We have the following observations on the impact of $T$, regardless of the choices for $\delta$ and $\epsilon$:
\begin{itemize}
    \item As $T$ increases, the number of samples required by each participant approach increases.
    \item However, the speeds of growth differ dramatically: Both \textbf{Regular Meter} and \textbf{Incremental Meter} grow much more slowly than \textbf{Resampling}, and moreover, \textbf{Incremental Meter} grows much more slowly than \textbf{Regular Meter}. In fact, based on Equations~\ref{eq:resampling},~\ref{eq:uniform:regular}, and~\ref{eq:uniform:incremental}, we can show that the sample size required by \textbf{Resampling}, \textbf{Regular Meter}, and \textbf{Incremental Meter} are $O(T\ln T)$, $O(T)$, and $O(\ln T)$, respectively, with respect to $T$.\footnote{We need to apply Stirling's approximation (to Equation~\ref{eq:uniform:incremental}) to obtain the $O(\ln T)$ result for \textbf{Incremental Meter}.}
    \item Although both \textbf{Incremental Meter} and \textbf{No Adaptivity} grow at the rate of $O(\ln T)$, there is still visible gap between them, which indicates opportunity for further improvement.
\end{itemize}

\vspace{0.5em}
\noindent
\textbf{(Impact of $\delta$)} By comparing Figure~\ref{fig:uniform-delta5-epsilon5} and Figure~\ref{fig:uniform-delta1-epsilon5}, where we keep $\epsilon=0.05$ unchanged but vary $\delta$ from 0.05 to 0.01, we observe that the sample complexity only slightly increases (not quite noticible given that the $y$-axis is in log scale).
Comparing Figure~\ref{fig:uniform-delta5-epsilon1} and Figure~\ref{fig:uniform-delta1-epsilon1} leads to the same observation.
This is understandable, as the impact of $\delta$ on the sample complexity of all participant approaches is (the same) $O(\ln\frac{1}{\delta})$.

\vspace{0.5em}
\noindent
\textbf{(Impact of $\epsilon$)} On the other hand, the impact of $\epsilon$ on the sample complexity of all participant approaches is much more significant. This can be evidenced by comparing Figure~\ref{fig:uniform-delta5-epsilon5} and Figure~\ref{fig:uniform-delta5-epsilon1}, where we keep $\delta=0.05$ but change $\epsilon$ from 0.05 to 0.01.
As we can see, the sample complexity increases by around 25$\times$!
This observation remains true if we compare Figure~\ref{fig:uniform-delta1-epsilon5} and Figure~\ref{fig:uniform-delta1-epsilon1}.
The rationale is simple -- the impact of $\epsilon$ on the sample complexity is (the same) $O(\frac{1}{\epsilon^2})$ for all approaches.

\vspace{0.5em}
\noindent
\textbf{(Impact of $m$)} Figure~\ref{fig:uniform-varying-m} further compares the sample complexity of the uniform meters when fixing $\delta=0.01$ and $\epsilon=0.01$. Figures~\ref{fig:uniform-m5} and~\ref{fig:uniform-m10} depict results when $m=5$ and $m=10$, respectively. (The $y$-axis is in log scale.) We see that sample complexity increases for both \textbf{Regular Meter} and \textbf{Incremental Meter}.
In fact, we can show that both meters actually grow at the (same) rate $O(\ln m)$ with respect to $m$.\footnote{Again, we need to apply Stirling's approximation (to Equation~\ref{eq:uniform:incremental}) to obtain the $O(\ln m)$ result for \textbf{Incremental Meter}.}

\begin{figure}
\centering
	\subfigure[ $m=5$]{
        \includegraphics[clip, trim=1cm 2cm 1cm 1cm, width=0.44\textwidth]{figures/uniform-delta1-epsilon1.pdf}
        \label{fig:uniform-m5}
	}
	\subfigure[$m=10$]{
        \includegraphics[clip, trim=1cm 2cm 1cm 1cm, width=0.44\textwidth]{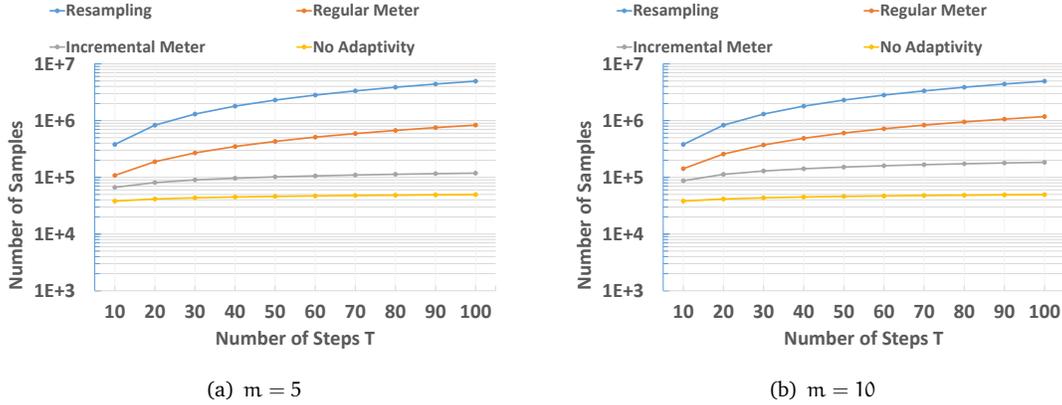}
        \label{fig:uniform-m10}
	}
\caption{Comparison of uniform meters with different $m$'s when setting $\delta=0.01$ and $\epsilon=0.01$, and varying $T$ from 10 to 100.}
\label{fig:uniform-varying-m}
\end{figure}

\subsection{Optimizations for Meters}\label{sec:evaluation:optimizations}

We next evaluate the effectiveness of the optimization techniques presented in Section~\ref{sec:optimization}: (1) nonuniform $\epsilon$'s and (2) multitenancy.

\subsubsection{Nonuniform Error Tolerance}

Figure~\ref{fig:nonuniform-meters} compares nonuniform meters when setting $(\epsilon_1,...,\epsilon_5)$ $=$ (0.01,0.02,0.03,0.04,0.05) with uniform ones.
For uniform meters and the two baselines \textbf{Resampling} and \textbf{No Adaptivity}, we set $\epsilon=\epsilon_1=0.01$.
We observe the following regardless of $\delta$:
\begin{itemize}
    \item The nonuniform regular meter (``\textbf{Regular, Nonuniform}'') slightly improves over \textbf{Regular Meter}, but not much.
    This is not surprising, though, if we compare Equation~\ref{eq:nonuniform:approx} with Equation~\ref{eq:nonuniform} --- for a given $m$, the sample complexity of both meters grows at the rate $O(T)$.
    \item On the other hand, the nonuniform incremental meter (``\textbf{Incremental, Nonuniform}'') significantly improves over its uniform counterpart \textbf{Incremental Meter}. Again, we can verify this by comparing Equation~\ref{eq:nonuniform-incremental:approx} with Equation~\ref{eq:nonuniform-incremental} --- the uniform version has sample complexity $O(m\ln T)$ whereas the nonuniform version has sample complexity $O(\ln T)$.
    \item The sample complexity of the nonuniform incremental meter is close to that of \textbf{No Adaptivity} (i.e., the ideal case). In fact, their sample complexity would be the same if we assume $\epsilon_1$-dominance for the nonuniform incremental meter. To verify, compare Equation~\ref{eq:nonuniform-incremental:approx} with Equation~\ref{eq:independent}.
\end{itemize}
In our experiments, we have tested other settings for $(\epsilon_1,...,\epsilon_5)$ and the observations remain valid.

\begin{figure}
\centering
	\subfigure[ $\delta=0.05$]{
        \includegraphics[clip, trim=1cm 2cm 1cm 1cm, width=0.44\textwidth]{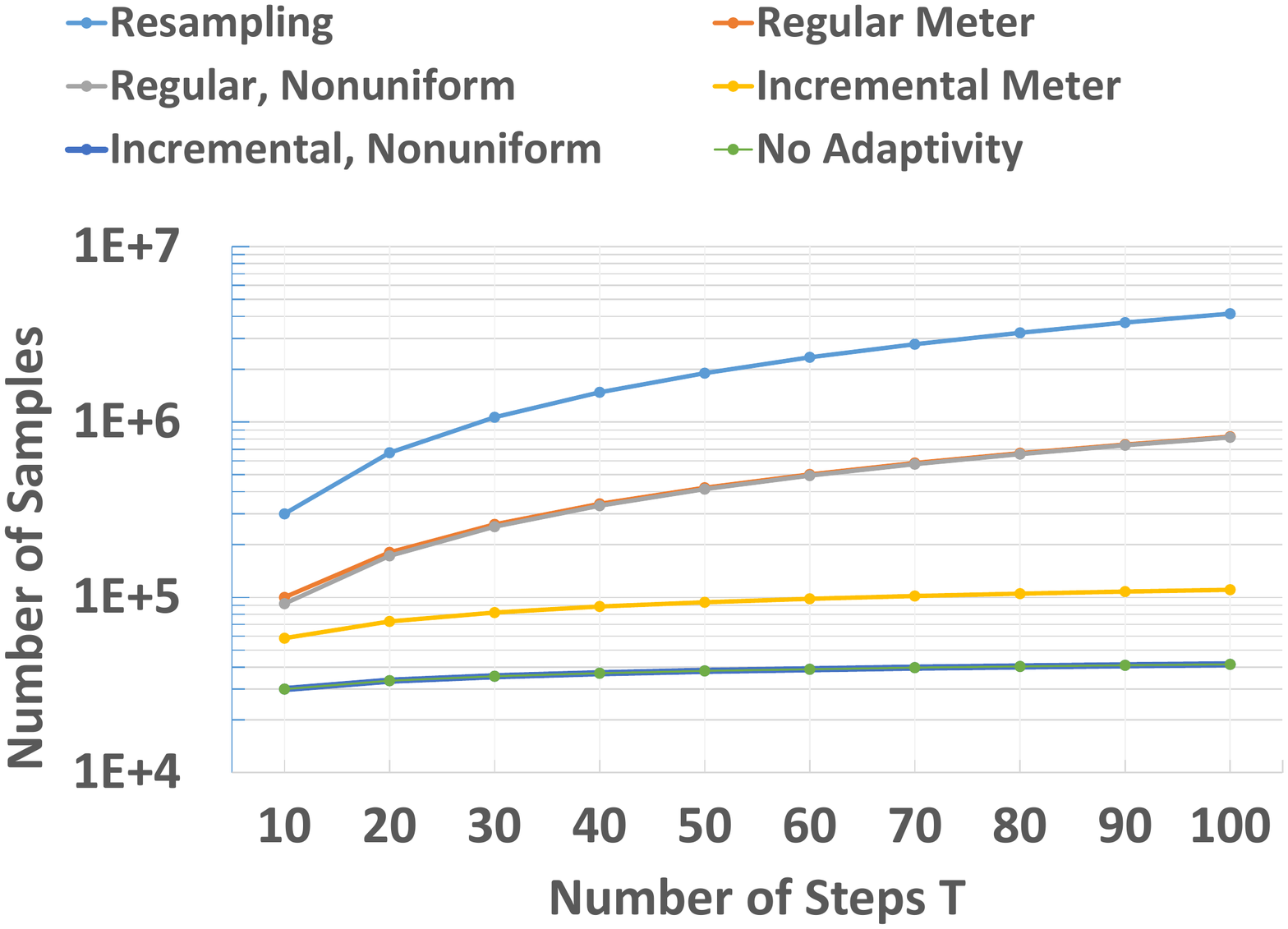}
        \label{fig:nonuniform-delta5}
	}
	\subfigure[$\delta=0.01$]{
        \includegraphics[clip, trim=1cm 2cm 1cm 1cm, width=0.44\textwidth]{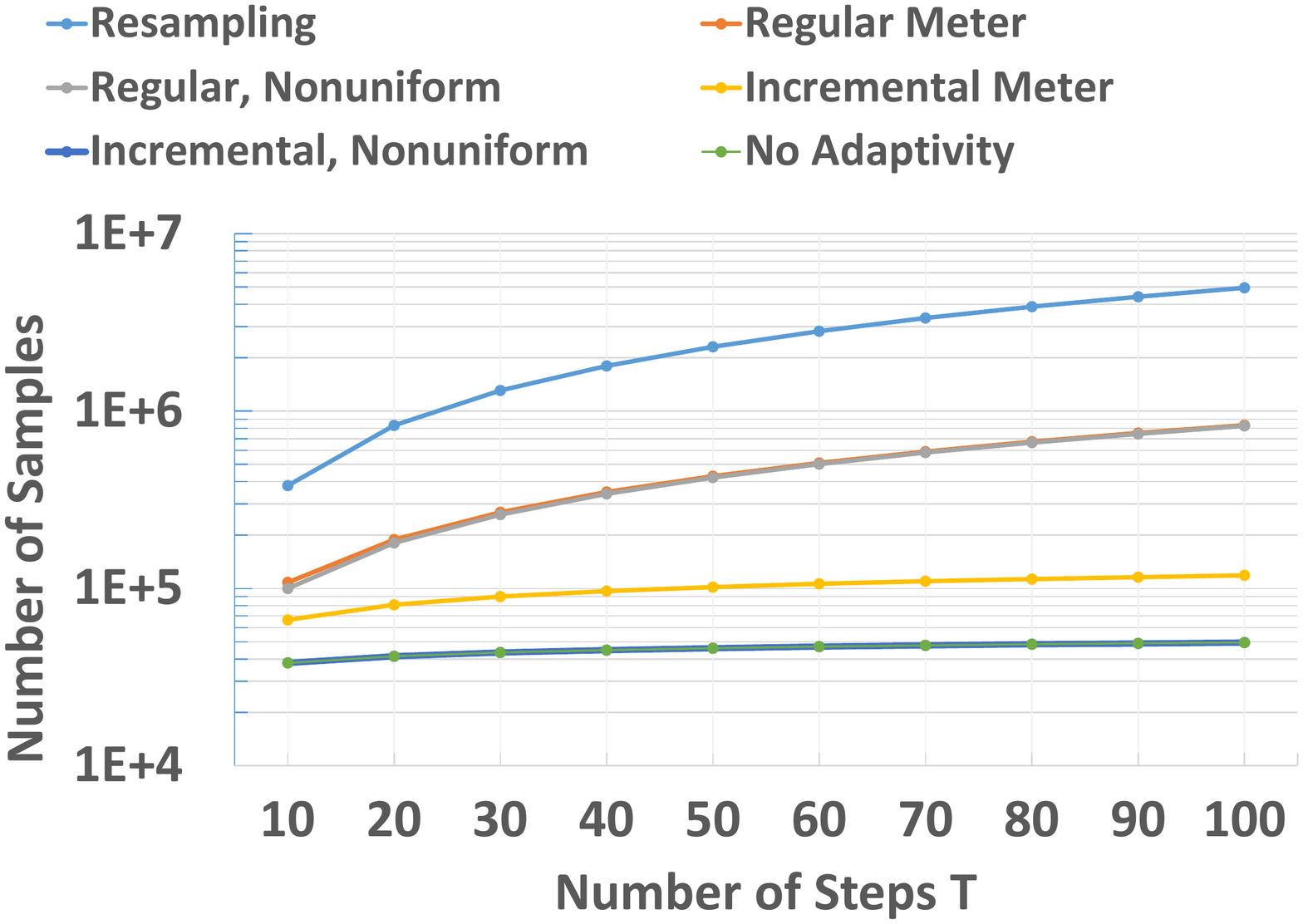}
        \label{fig:nonuniform-delta1}
	}
\caption{Comparison of nonuniform meters with uniform ones when setting $(\epsilon_1,...,\epsilon_5)=(0.01,0.02,0.03,0.04,0.05)$.}
\label{fig:nonuniform-meters}
\end{figure}

\subsubsection{Multitenancy}\label{sec:evaluation:optimizations:multitenancy}

We next evaluate the sample complexity of the multitenancy setting (specifically, two-tenancy setting) presented in Section~\ref{sec:optimization:multitenancy}.
We focus on nonuniform meters by setting $(\epsilon_1,...,\epsilon_5)$ $=$ $(0.01,0.02,0.03,0.04,0.05)$ and $\delta=0.01$.

Figure~\ref{fig:two-tenancy-meters} presents the results.
We observe that the two-tenancy regular meter further improves significantly over its single-tenancy counterpart.
In fact, it roughly improves by 2$\times$, which can be verified by comparing Equation~\ref{eq:two-tenancy:regular:approx} with Equation~\ref{eq:nonuniform:approx}.
In contrast, the two-tenancy incremental meter does not improve over the single-tenancy one.
This is understandable, though, if we look at Equation~\ref{eq:two-tenancy:incremental:approx} and Equation~\ref{eq:nonuniform-incremental:approx} --- they are just the same.
In fact, both of them have matched the sample complexity of \textbf{No Adaptivity} (i.e., Equation~\ref{eq:independent}), which represents the ideal case where all submissions are independent of each other.

\begin{figure}[t]
\centering
\includegraphics[clip, trim=1cm 2cm 1cm 1cm, width=0.44\textwidth]{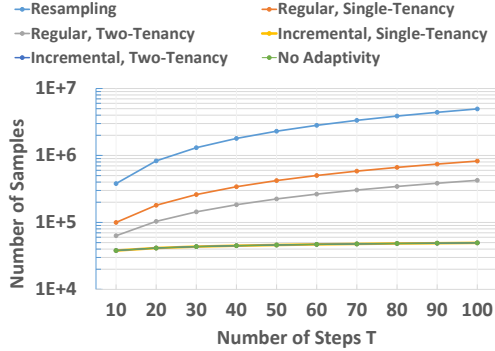}
\caption{Comparison of two-tenancy vs. single-tenancy meters, $(\epsilon_1,...,\epsilon_5)$ $=$ (0.01,0.02,0.03,0.04,0.05) and $\delta=0.01$.}
\label{fig:two-tenancy-meters}
\end{figure}

\subsubsection{``Time Travel''}\label{sec:evaluation:optimizations:time-travel}

We further evaluate the sample complexity of the ``time travel'' setting presented in Section~\ref{sec:optimization:time-travel} where users are allowed to take reversions during development history.
Again, we focus on nonuniform meters by setting $(\epsilon_1,...,\epsilon_5)$ $=$ $(0.01,0.02,0.03,0.04,0.05)$ and $\delta=0.01$.
In our evaluation, we tested different settings for $T$ and $B$.
For a given $T$, we varied the ``time travel'' budget $B$ from 1 to $T$, and set $(t_1,...,t_B)$ $=$ $(1, ..., B)$.

\begin{figure}
\centering
	\subfigure[ $T=5$]{
        \includegraphics[clip, trim=1cm 2cm 1cm 1cm, width=0.44\textwidth]{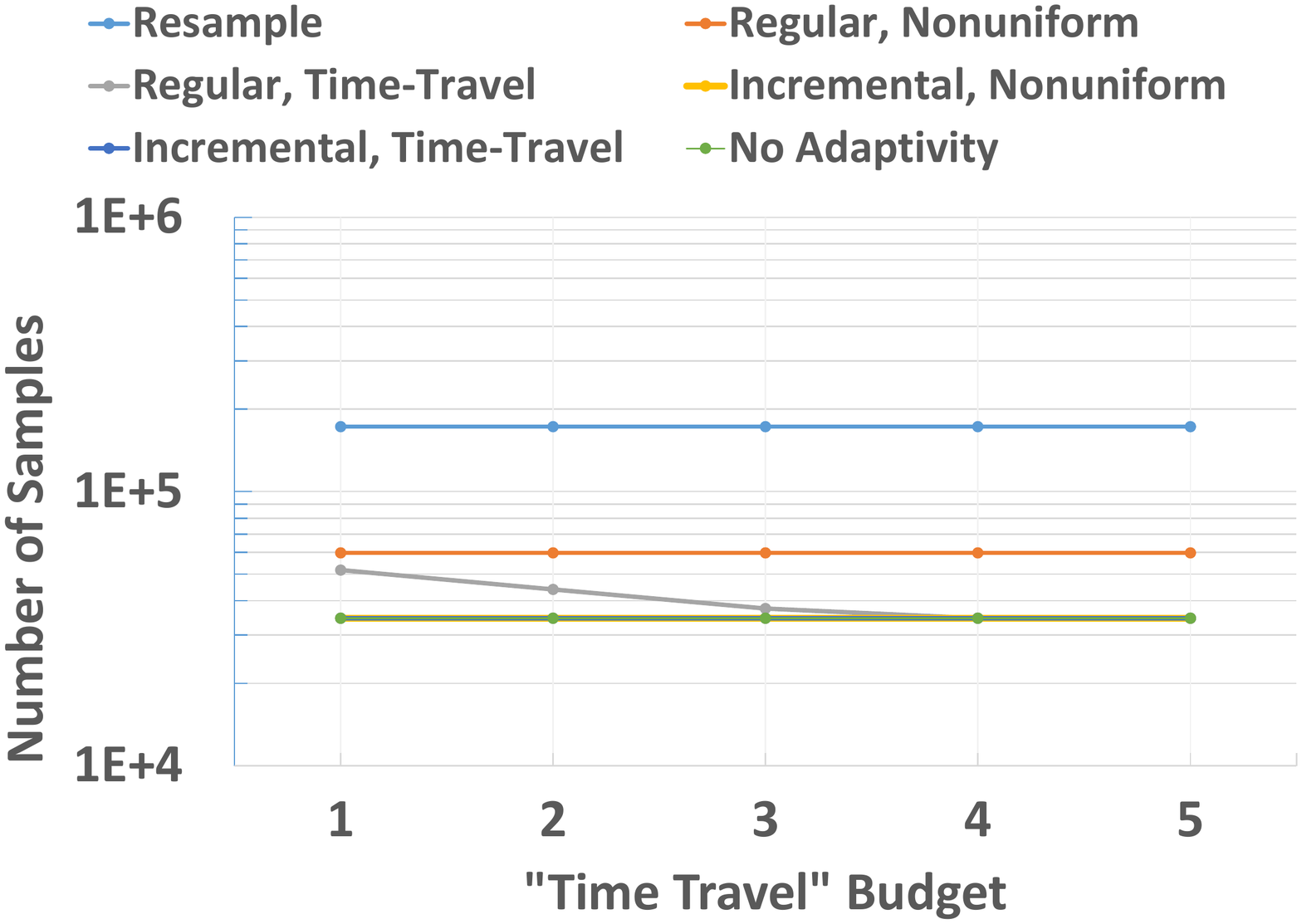}
        \label{fig:time-travel-T5}
	}
	\subfigure[$T=10$]{
        \includegraphics[clip, trim=1cm 2cm 1cm 1cm, width=0.44\textwidth]{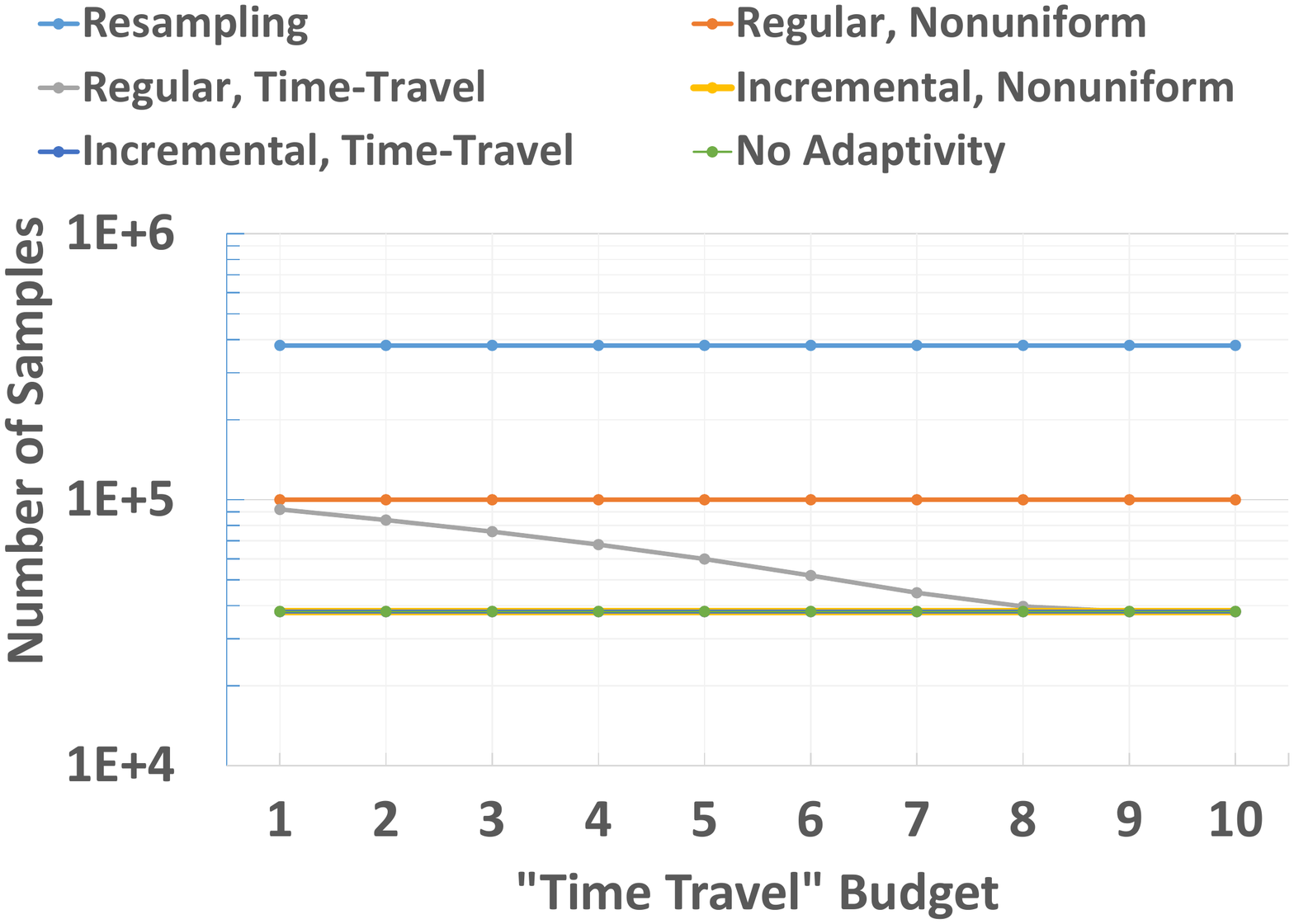}
        \label{fig:time-travel-T10}
	}
\caption{Comparison of meters with and without ``time travel,'' setting $(\epsilon_1,...,\epsilon_5)$ $=$ $(0.01,0.02,0.03,0.04,0.05)$ and $\delta=0.01$.}
\label{fig:time-travel-meters}
\end{figure}

Figure~\ref{fig:time-travel-meters} summarizes the results when setting (1) $T=5$; and (2) $T=10$.
We observe that the ``time travel'' regular meter further improves significantly over its normal counterpart.
Moreover, the improvement increases as we increase the ``time travel'' budget.
In fact, the sample complexity converges to its lower bound -- the (ideal) sample complexity of \textbf{No Adaptivity} (i.e., Equation~\ref{eq:independent}) -- as $B$ increases to $T$.
In contrast, the ``time travel'' incremental meter does not improve over the normal one.
This is understandable, though, as both of them have already matched the lower bound.

\subsection{Meter In Action: A Revisit} \label{sec:evaluation:case-studies}

We now revisit the two ML applications presented in Section~\ref{sec:sys:use-cases} and study the sample complexity if \meter~were integrated into their development lifecycles, under various parameter settings. We set $T=8$ as both applications involve 8 development steps.
For nonuniform meters, we further set $(\epsilon_1,...,\epsilon_5)=(0.01,0.02,0.03,0.04,0.05)$, and we set $\epsilon=\epsilon_1=0.01$ for the other approaches.
We vary $\delta$ from 0.1 to 0.01, which corresponds to varying reliability/confidence $1-\delta$ from 0.9 to 0.99.

\begin{figure}[t]
\centering
\includegraphics[clip, trim=1cm 2cm 1cm 1cm, width=0.42\textwidth]{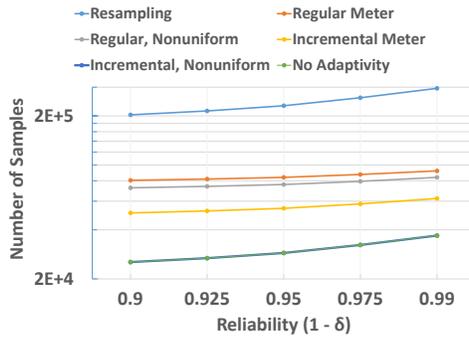}
\caption{Integration of \meter~into real-world ML application development lifecycles in Section~\ref{sec:sys:use-cases}.}
\label{fig:case-study-meters}
\end{figure}

Figure~\ref{fig:case-study-meters} summarizes the sample complexity of various approaches.
As we can see, the \textbf{Resampling} approach is perhaps too expensive for practical use, as it requires around 200K to 300K human labels even for a short development history with only 8 steps.
All versions of \meter~can significantly improve sample complexity over \textbf{Resampling}, by 2.5$\times$ to 8$\times$.
Moreover, the incremental meters outperform the regular meters, by 1.6$\times$ to 3$\times$.
The most significant improvement comes from the nonuniform incremental meter, whose sample complexity is close to that of the ideal case represented by \textbf{No Adaptivity}: For example, it would require only 25K labels when setting reliability to 0.9; even when reliability increases to 0.99, it would require just 37K labels.
These observations are in line with what we have observed in previous evaluations, and they demonstrate the practicality of integrating the meters into real-world ML application development activities.

\section{Related Work}\label{sec:related-work}

\paragraph*{AutoML Systems}

There is a flurry of recent work on developing automatic machine learning (AutoML) systems that aims for alleviating development effort by providing ``declarative'' machine learning services. In a typical AutoML system, users only need to upload their datasets and provide high-level specifications of their machine learning tasks (e.g., schemata of inputs/outputs, task categories such as binary classification/multi-class classification/regression, loss functions to be minimized, etc.), and the system can take over the rest via automated pipeline execution (e.g., training/validating/testing), infrastructure support (e.g., resource allocation, job scheduling), and performance-critical functionality such as model selection and hyperparameter tuning. Examples of such systems include industrial efforts made by major cloud service providers such as Amazon~\cite{AmazonSageMaker}, Microsoft~\cite{AzureML}, and Google~\cite{googleCloudML}, as well as ones from the academia, such as the Northstar system developed at MIT~\cite{Northstar} and our own recent effort on the \eml service~\cite{KarlasLWZ18,LiZLWZ18,ZhangWL17}. Our work in this paper is in line with but orthogonal to these AutoML systems: While they have focused on automation (and therefore \emph{efficiency}) of the ML application development cycle itself, we seek an automated way of verifying the \emph{efficacy} of the ML application developed.

\paragraph*{Adaptive Analysis}
Recent theoretical work has revealed the increased risk of drawing false conclusion (known as ``false discovery'' in the literature) via statistical methods, due to the presence of adaptive analysis~\cite{DworkFHPRR15}.
Moreover, it has been shown that achieving statistical validity is often computationally intractable in adaptive settings~\cite{HardtU14,SteinkeU15}.
Intuitively, adaptive analysis makes it more likely to draw conclusions tied to the specific data set used in a statistical study, rather than conclusions that can be generalized to the underlying distribution that governs data generation.
Our contribution in this paper can be viewed as an application of this theoretical observation to a novel, important scenario emerging from the field of ML application development lifecycle management and quality control.
\section{Conclusion}\label{sec:conclusion}

We have presented \meter, an overfitting management system for modern, continuous ML application development.
We discussed its system architecture, design principles, as well as implementation details.
We focused on addressing the primary challenge that could prevent \meter~from being practical, i.e., the sample complexity of the test set in the presence of adaptive analysis.
We further proposed various optimization techniques that can significantly take down the required amount of test labels, by applying recent developments from the theory community.
Evaluation results demonstrate that \meter~can outperform resampling-based approaches with test set sizes that are an order of magnitude smaller, while providing the same, stringent probabilistic guarantees of the overfitting signals.

\clearpage
\bibliography{references}
\bibliographystyle{abbrv}

\clearpage
\appendix
\section{Theoretical Results}

This section includes theoretical results referenced by the paper, the details of which have been omitted due to space limitation.

\subsection{Hoeffding's Inequality}
\label{appendix:theory:hoeffiding}

Throughout this paper, we use the following form of the Hoeffding's inequality.
(A more general form is in Lemma 4.5 of~\cite{UMLbook}.)

\begin{theorem}[Hoeffding's Inequality]
\label{thm:hoeffding}
For a bounded function $0\leq h(x) \leq 1$ on the i.i.d. random variables $x_1,...,x_n \in \mathcal{X} $  and for all $\epsilon > 0 $ we have 
\begin{equation}
\Pr \left[  \left| \frac{1}{n} \sum_{i} h(x_i)- \mathbb{E}[h(x)] \right|   > \epsilon \right] < 2\exp \left(  -2n \epsilon^2 \right). 
\end{equation}
\end{theorem}

\subsection{Basics for Sample Complexity Analysis}
\label{appendix:theory:sample-complexity-basics}

Here we cover basics for sample complexity analysis. In particular, we present analysis for two basic cases: (1) one single submission; and (2) multiple, independent submissions. For comparison purpose, we also include analysis for the \textbf{Resampling} baseline.

\subsubsection{The Case of One Single Model}
\label{appendix:theory:sample-complexity-basics:single}

\begin{definition}
We say that $D_{test}$ overfits by $\epsilon$ with respect to model $H$ and performance measure $l$, if and only if
$|\Delta_H| > \epsilon.$
\end{definition}
We can simply leverage the Hoeffding's inequality (Appendix~\ref{appendix:theory:hoeffiding}) to derive $|D_{test}|$, assuming data points in $D_{test}$ are i.i.d. samples from $\mathcal{D}_{test}$. Specifically, 
\begin{theorem}\label{theorem:one-model}
Suppose that we use a loss function $l$ bounded by $[0, 1]$.
Given a model $H\in\mathcal{H}$, the required test set size satisfies
\begin{equation}
\Pr [  |\Delta_H|  > \epsilon] < 2\exp \left(  -2|D_{test}| \epsilon^2 \right) < \delta.
\end{equation}
As a result, it follows that 
$|D_{test}| > \frac{\ln(2/\delta)}{2\epsilon^2}.$
\end{theorem}

\begin{proof}
Suppose that we use a loss function $l$ bounded by $[0, 1]$.
Given a model $H\in\mathcal{H}$, we have
$$\Pr \left[  \left|\frac{1}{n} \sum_{x_i\in D_{test}} l(H, x_i)  - \mathbb{E}[l(H, x)] \right| > \epsilon \right] < 2\exp \left(  -2n \epsilon^2 \right).$$
Since 
$$l(H, D_{test})=\frac{1}{n} \sum_{x_i\in D_{test}} l(H, x_i)$$
and 
$$l(H,\mathcal{D}_{test})=\mathbb{E}\left[l(H, x) \right],$$
it follows that
$$\Pr \left[  \left|l(H, D_{test})  - l(H,\mathcal{D}_{test}) \right|  > \epsilon \right] < 2\exp \left(  -2n \epsilon^2 \right),$$
where $n=|D_{test}|.$ Setting $2\exp \left(  -2n \epsilon^2 \right)<\delta$ yields
$$|D_{test}| > \frac{\ln(2/\delta)}{2\epsilon^2}.$$
This completes the proof of the theorem.
\end{proof}

For example, if the user sets $\epsilon=0.1$ and $\delta=0.05$, then the test set needs at least 185 data examples.
On the other hand, if the user wishes a more fine-grained overfitting indicator with higher confidence, e.g., $\epsilon=0.01$ and $\delta=0.01$, then the size of the test set would blow up to 26,492.

\subsubsection{The Case of Independent Models}
\label{appendix:theory:sample-complexity-basics:independent}

Suppose that all versions submitted by the developer for testing are \emph{independent} (i.e., the next submission does not depend on the overfitting signal returned by \meter~for the present submission).
Then there are just $T$ possible submissions in $\mathcal{H}^{(T)}$.
By Definition~\ref{definition:history}, applying the union bound (combined with the Hoeffding's inequality) we obtain
\begin{equation*}
\Pr [\exists H \in\mathcal{H}^{(T)},~~|\Delta_H|  > \epsilon] < 2T\exp \left(  -2|D_{test}| \epsilon^2 \right) < \delta,
\end{equation*}
where $\mathcal{H}=\{H_1, ..., H_T\}$.
The required $|D_{test}|$ is then simply
\begin{equation*} 
     | D_{test}|  > \frac{\ln(2T/\delta)}{2\epsilon^2}.
\end{equation*}
For instance, if we set $\epsilon=0.01$ and $\delta=0.01$, we need 38K examples to test $T=10$ (independent) submissions.

\vspace{-0.5em}
\subsubsection{The Resampling Baseline}
\label{appendix:theory:sample-complexity-basics:re-sampling}

Let the test set in the $t$-th submission be $D_{test}^{(t)}$, for $1\leq t\leq T$.
We require $|D_{test}^{(t)}|=\cdots=|D_{test}^{(T)}|=n$, and we are interested in the total number of test examples 
$$|D_{test}| = |D_{test}^{(t)}| + ...+ |D_{test}^{(T)}|=T\cdot n.$$ 
Again, applying the union bound (combined with the Hoeffding's inequality) we obtain
\begin{equation*}
\Pr [\exists H \in\mathcal{H}^{(T)},~~|\Delta_H|  > \epsilon] < 2T\exp \left(  -2n \epsilon^2 \right) < \delta,
\end{equation*}
where $\mathcal{H}=\{H_1, ..., H_T\}$. It then follows that
\begin{equation*}
    |D_{test}| = T\cdot n > T\cdot\frac{\ln(2T/\delta)}{2\epsilon^2}.
\end{equation*}
Use the same setting where $\epsilon=0.01$ and $\delta=0.01$, it would now require 380K examples to test just $T=10$ models.

\subsection{Proofs of Theorems}\label{appendix:theory:proofs}



\subsubsection{Proof of Theorem~\ref{theorem:uniform}}
\label{appendix:theory:proofs:uniform}
\begin{proof}
Following the idea when proving the Ladder mechanism~\cite{blum2015ladder}, let $\mathcal{T}$ be the tree that captures the dependencies between submissions (in the adaptive setting). It then follows that
$$\Pr [\exists H \in\mathcal{H}^{(T)},~~|\Delta_H|  > \epsilon] < 2|\mathcal{T}|\exp \left(  -2|D_{test}| \epsilon^2 \right),$$
using the union bound over all submissions in $\mathcal{T}$ and applying the Hoeffding's inequality for each submission.

For the uniform, regular meter, $\mathcal{T}$ is a \emph{perfect} $m$-ary tree with $T$ levels (except for the root that represents $H_0$), where each internal tree node has exactly $m$ children.
Therefore, 
$$|\mathcal{T}|=\sum_{t=1}^{T}m^t=\frac{m(m^T-1)}{m-1}=|\mathcal{T}(m,T)|.$$
This completes the proof of the theorem.
\end{proof}

\subsubsection{Proof of Theorem~\ref{theorem:uniform-incremental}}
\label{appendix:theory:proofs:uniform-incremental}

\begin{proof}
Again, the idea is to compute the tree size $|\mathcal{T}|$.
Due to additional constraints in the incremental meter, $|\mathcal{T}|\leq |\mathcal{T}(m,T)|$ in Theorem~\ref{theorem:uniform}.
Specifically, we use $h(k, t)$ to represent the number of tree nodes with value $k$ and depth $t$, for $1\leq k\leq m$ and $1\leq t\leq T$.
It then follows that
\begin{equation}\label{eq:tree-size}
  |\mathcal{T}|=\sum\nolimits_{k=1}^{m}\sum\nolimits_{t=1}^{T} h(k, t).  
\end{equation}
We now divide the whole proof procedure into two phases: (I) derive $h(k,t)$ and (II) derive $|\mathcal{T}|$ based on Equation~\ref{eq:tree-size}.

\vspace{0.5em}
\noindent
\textbf{\underline{I. Derive $h(k,t)$.}} We have the following observations:
\begin{align*}
    h(k,1) = 1 ={k-1 \choose 0}, &\quad &\forall k\in[m];\\
    h(k,2) = k = {k \choose 1},&\quad & \forall k\in[m];\\
    h(k,3) = k\cdot\frac{(k+1)}{2} = {k + 1 \choose 2},&\quad & \forall k\in[m];\\
    h(k,4) = k\cdot\frac{(k+1)}{2}\cdot\frac{(k+2)}{3}={k + 2 \choose 3},&\quad & \forall k\in[m].
\end{align*}
In general, using induction (on $t$) we can prove that
\begin{equation}\label{eq:h_k_t}
  h(k,t)={k+t-2 \choose t-1}.  
\end{equation}
To see this, suppose that $h(k,t)={k+t-2 \choose t-1}$ for all $1\leq t\leq n$.
Now consider $t=n+1$.
Observe that $h(k,t+1)=\sum_{i=1}^k h(i,t)$ for all $t\geq 1$,
as at the step $t+1$ we can grow $\mathcal{T}^{(t)}$ -- the snapshot of $\mathcal{T}$ at time $t$ -- by appending nodes with value $k$ only behind nodes with values $i\leq k$ in $\mathcal{T}^{(t)}$.
As a result,
\begin{equation}\label{eq:induction}
    h(k,n+1)=\sum_{i=1}^k h(i, n)=\sum_{i=1}^k {i + n-2 \choose n-1}.
\end{equation}
We can get a closed-form solution to $h(k, n+1)$ by using the recursive definition of the binomial coefficient: 
\begin{align}
\label{eq:binomcoeff_recursive}
 {n \choose k} = {n-1 \choose k-1} + {n-1 \choose k}.
\end{align}

Specifically, writing Equation~\ref{eq:induction} explicitly and further noticing that ${1 + n - 2 \choose n}=0$, we obtain
\begin{eqnarray*}
    (\ref{eq:induction}) 
    &=& {1 + n - 2 \choose n} + {1 + n-2 \choose n-1} + \cdots + {k + n-2 \choose n-1}\\
    &=& {k + n-1 \choose n}.
\end{eqnarray*}
This finishes the induction on proving Equation~\ref{eq:h_k_t}.

\vspace{0.5em}
\noindent
\textbf{\underline{II. Derive $|\mathcal{T}|$.}} By Equations~\ref{eq:tree-size} and~\ref{eq:h_k_t} we have
$$|\mathcal{T}|=\sum\nolimits_{k=1}^{m}\sum\nolimits_{t=1}^{T} {k+t-2 \choose t-1}.$$
Define 
\begin{equation}\label{eq:I_k_T:def}
    |\mathcal{T}^{(k)}(m,T)|=\sum\nolimits_{t=1}^{T} {k+t-2 \choose t-1}.
\end{equation}
Using the same technique as in proving Equation~\ref{eq:induction}, we can prove
\begin{equation}\label{eq:I_k_T}
  |\mathcal{T}^{(k)}(m,T)|={k+T-1 \choose T-1} = {k+T-1 \choose k}.  
\end{equation}
Specifically, notice that 
${k + 2 - 2 \choose 0}={k + 1 - 2 \choose 0} = 1.$
It follows that
\begin{eqnarray*}
    (\ref{eq:I_k_T:def}) & = & {k+1-2 \choose 0} + {k+2-2 \choose 1} + \cdots + {k+T-2 \choose T-1} \\
    &=& {k+2-2 \choose 0} + {k+2-2 \choose 1} + \cdots + {k+T-2 \choose T-1} \\
    &=& {k+T-1 \choose T-1}.
\end{eqnarray*}
Applying the same technique once again, we can show that
\begin{equation} \label{eq:tree-size-uniform-incremental}
    |\mathcal{T}| = {m+T \choose m} - 1.
\end{equation}
Specifically, notice that ${1+T - 1 \choose 0} = 1$. Let $S=|\mathcal{T}|+1$. We have
\begin{eqnarray*}
    S &=& {1+T - 1 \choose 0} + \sum\nolimits_{k=1}^{m} {k+T-1 \choose k}\\
    &=& {1+T - 1 \choose 0} + {1+T - 1 \choose 1} + \cdots + {m+T - 1 \choose m}\\
    &=& {m+T \choose m}.
\end{eqnarray*} 
This finishes the proof of Theorem~\ref{theorem:uniform-incremental}.
\end{proof}

\subsubsection{Proof of Corollary~\ref{corollary:nonuniform-regular}}
\label{appendix:theory:proofs:nonuniform}

\begin{proof}
We need to count the number of tree nodes in $\mathcal{T}$ for each different $k\in[m]$.
By symmetry, each $k$ contributes equally to the number of tree nodes. Therefore, the number of tree nodes for each $k$ is simply (the same)
$|\mathcal{T}(m,T)|/m$.
\end{proof}

\subsubsection{Proof of Corollary~\ref{corollary:nonuniform-incremental}}
\label{appendix:theory:proofs:nonuniform-incremental}

\begin{proof}
Similarly, we count the number of tree nodes in $\mathcal{T}$ for each individual $k\in[m]$.
Following the proof of Theorem~\ref{theorem:uniform-incremental} in Appendix~\ref{appendix:theory:proofs:uniform-incremental}, it is easy to see that this count is simply $|\mathcal{T}^{(k)}(m,T)|={k+T-1 \choose k}$, i.e., Equation~\ref{eq:I_k_T}.
\end{proof}

\section{Multi-tenant Meter}
\label{appendix:multitenant-meter}

We can extend the results in the two-tenant setting to $l$ tenants. Under the assumption that every user submits an equal number of models $\frac{T}{l}$ we have 
\begin{align}
\delta &> l \sum\nolimits_{k=1}^{m} 2 \cdot\frac{m^{T/l}-1}{m-1} \cdot \exp(-2 | D_{test}| \epsilon_k^2), \\
\delta &> l \sum\nolimits_{k=1}^{m} 2 \cdot\binom{k + \frac{T}{l}-1}{k} \cdot \exp(-2 | D_{test}| \epsilon_k^2) 
\end{align}
for the regular and incremental meters, respectively. 
We can further generalize the results to an unequal number of model submissions among tenants. Spefifically, given tenants $i = 1,...,l$, each of whom contributes $T_i$ models successively, we have 
\begin{align*}
\delta &> \sum\limits_{i=1}^{l} \sum\limits_{k=1}^{m} 2 \cdot\frac{m^{T_i}-1}{m-1} \cdot \exp(-2 | D_{test}| \epsilon_k^2), \\
\delta &> \sum \limits_{i=1}^{l} \sum\limits_{k=1}^{m} 2 \cdot\binom{k + T_i-1}{k} \cdot \exp(-2 | D_{test}| \epsilon_k^2).
\end{align*}




\end{document}